\documentclass[usenames, dvipsnames]{article}

\newif\ifanonymize
\anonymizefalse

\newif\ifarxiv
\arxivtrue

\usepackage{amsmath, amsfonts, amsthm}
\usepackage{booktabs}
\usepackage{enumitem}
\usepackage{graphicx}
\usepackage{makecell}
\usepackage{mathtools}
\usepackage{microtype}
\usepackage{multirow}
\usepackage{nicefrac}
\usepackage{subcaption}
\usepackage{tabularx}
\usepackage{tikz}
\usepackage{url}
\usepackage{wrapfig}
\usepackage{natbib}
\usepackage{hyperref}
\usepackage[nameinlink]{cleveref}

\ifanonymize
\usepackage{icml2020}
\else
\usepackage[accepted]{icml2020}
\fi

\newcommand{\ourtitle}{Normalizing Flows on Tori and Spheres}

\icmltitlerunning{\ourtitle}

\newcommand{\R}{\mathbb{R}}

\newcommand{\sphere}[1]{\mathbb{S}^{#1}}
\newcommand{\torus}[1]{\mathbb{T}^{#1}}

\newcommand{\SU}[1]{\text{SU}(#1)}
\newcommand{\MM}{\mathcal{M}}
\newcommand{\NN}{\mathcal{N}}

\newcommand{\br}[1]{\mathopen{}\left(#1\right)\mathclose{}}
\newcommand{\set}[1]{\left\{#1\right\}}

\newcommand{\abs}[1]{\left|#1\right|}
\newcommand{\norm}[1]{\left\|#1\right\|}

\newcommand{\g}{\,|\,}

\newcommand{\pderiv}[2]{\frac{\partial{#1}}{\partial{#2}}}

\newcommand{\bigo}[1]{\mathcal{O}\br{#1}}

\newcommand{\Tcs}{T_{c\rightarrow s}}
\newcommand{\Tsc}{T_{s\rightarrow c}}

\allowdisplaybreaks

\newtheorem{defin}{Definition}
\newtheorem{prop}{Proposition}

\begin{document}

\twocolumn[
\icmltitle{\ourtitle}

\icmlsetsymbol{equal}{*}

\begin{icmlauthorlist}
\icmlauthor{Danilo Jimenez Rezende}{equal,dm}
\icmlauthor{George Papamakarios }{equal,dm}
\icmlauthor{S\'ebastien Racani\`ere}{equal,dm}
\icmlauthor{Michael S.\ Albergo}{nyu}
\icmlauthor{Gurtej Kanwar}{mit}
\icmlauthor{Phiala E.\ Shanahan}{mit}
\icmlauthor{Kyle Cranmer}{nyu}
\end{icmlauthorlist}

\icmlaffiliation{dm}{DeepMind, London, U.K.}
\icmlaffiliation{nyu}{New York University, New York, U.S.A.}
\icmlaffiliation{mit}{Center for Theoretical Physics, Massachusetts Institute of Technology, Cambridge, MA 02139, U.S.A} 

\icmlcorrespondingauthor{Danilo Jimenez Rezende}{danilor@google.com}
\icmlcorrespondingauthor{George Papamakarios}{gpapamak@google.com}
\icmlcorrespondingauthor{S\'ebastien Racani\`ere}{sracaniere@google.com}

\icmlkeywords{normalizing flows, directional statistics}

\vskip 0.3in
]

\printAffiliationsAndNotice{\icmlEqualContribution} 

\begin{abstract}
Normalizing flows are a powerful tool for building expressive distributions in high dimensions. So far, most of the literature has concentrated on learning flows on Euclidean spaces. Some problems however, such as those involving angles, are defined on spaces with more complex geometries, such as tori or spheres. In this paper, we propose and compare expressive and numerically stable flows on such spaces. Our flows are built recursively on the dimension of the space, starting from flows on circles, closed intervals or spheres.
\end{abstract}

\section{Introduction}
\label{sec:intro}
Normalizing flows are a flexible way of defining complex distributions on high-dimensional data. A \textit{normalizing flow} maps samples from a base distribution $\pi(u)$ to samples from a target distribution $p(x)$ via a transformation $f$ as follows:
\begin{equation}\label{eq:flow_sampling}
    x = f(u)\quad\text{where}\quad u\sim \pi(u).
\end{equation}
The transformation $f$ is restricted to be a \textit{diffeomorphism}: it must be invertible and both $f$ and its inverse $f^{-1}$ must be differentiable. This restriction allows us to calculate the target density $p(x)$ via a change of variables:
\begin{equation}\label{eq:flow_jac}
    p(x) = \pi\br{f^{-1}(x)}\abs{\det\br{\pderiv{f^{-1}}{x}}}.
\end{equation}
In practice, $\pi(u)$ is often taken to be a simple density that can be easily evaluated and sampled from, and either $f$ or its inverse $f^{-1}$ are implemented via neural networks such that the Jacobian determinant is efficient to compute.

A normalizing flow implements two operations: sampling via \Cref{eq:flow_sampling}, and evaluating the density via \Cref{eq:flow_jac}. These operations have distinct computational requirements: generating samples and evaluating their density requires only $f$ and its Jacobian determinant, whereas evaluating the density of arbitrary datapoints requires only $f^{-1}$ and its Jacobian determinant. Thus, the intended usage of the flow dictates whether $f$, $f^{-1}$ or both must have efficient implementations. For an overview of various implementations and associated trade-offs, see \citep{papamakarios2019normalizing}.

In most existing implementations of normalizing flows, both $u$ and $x$ are defined to live in the Euclidean space $\R^D$, where $D$ is determined by the data dimensionality. 
However, this Euclidean formulation is not always suitable, as some datasets are defined on spaces with non-Euclidean geometry. For example, if $x$ represents an angle, its `natural habitat' is the $1$-dimensional circle; if $x$ represents the location of a particle in a box with periodic boundary conditions, $x$ is naturally defined on the $3$-dimensional torus.

The need for probabilistic modelling of non-Euclidean data often arises in applications where the data is a set of angles, axes or directions \citep{mardia2009directional}. Such applications include protein-structure prediction in molecular biology \citep{hamelryck2006sampling, mardia2007protein, boomsma2008generative, shapovalov2011smoothed}, rock-formation analysis in geology \citep{peel2001mixtures}, and path navigation and motion estimation in robotics \citep{feiten2013rigid, senanayake2018directional}.
Non-Euclidean spaces have also been explored in machine learning, and specifically generative modelling, as latent spaces of variational autoencoders \citep{davidson2018svae, falorsi2018explorations, wang2019riemannian, mathieu2019poincare, wang2019wasserstein}.

A more general formulation of normalizing flows that is suitable for non-Euclidean data is to take $u\in\MM$ and $x\in\NN$, where $\MM$ and $\NN$ are \textit{differentiable manifolds} and $f:\MM\rightarrow\NN$ is a diffeomorphism between them. A difficulty with this formulation is that $\MM$ and $\NN$ are \textit{diffeomorphic} by definition, so they must have the same topological properties \citep{kobayashi1963foundations}. To circumvent this restriction, \citet{gemici2016normalizing} first project $\MM$ to $\R^D$, apply the usual flows there, and then project $\R^D$ back to $\NN$.
However, a naive application of this approach can be problematic when $\MM$ or $\NN$ are not diffeomorphic to $\R^D$; in this case, the projection maps will necessarily contain singularities (i.e.~points with non-invertible Jacobian), which may result in numerical instabilities during training. Another approach, proposed by \citet{falorsi2019reparameterizing} for the case where $\NN$ is a Lie group, is to apply the usual Euclidean flows on the Lie algebra of $\NN$ (i.e.~the tangent space at the identity element), and then use the exponential map to project the Lie algebra onto $\NN$.
Since these maps are not from $\NN$ to itself, they are not straightforward to compose. Also, the exponential map is not bijective and is computationally expensive in general. We discuss these issues in more detail in \Cref{sec:related_work}.

\textbf{Our main contribution} is to propose a new set of expressive normalizing flows for the case where $\MM$ and $\NN$ are compact and connected differentiable manifolds. Specifically, we construct flows on the $1$-dimensional circle $\sphere{1}$, the $D$-dimensional torus $\torus{D}$, the $D$-dimensional sphere $\sphere{D}$, and arbitrary products of these spaces.
The proposed flows can be made arbitrarily flexible, and avoid the numerical instabilities of previous approaches.
Our flows are applicable when we already know the manifold structure of the data, such as when the data is a set of angles or latent variables with a prescribed topology. We empirically demonstrate the proposed flows on synthetic inference problems designed to test the ability to model sharp and multi-modal densities.

\section{Methods}
\label{sec:methods}
We will begin by constructing expressive and numerically stable flows on the circle $\sphere{1}$. Then, we will use these flows as building blocks to construct flows on the torus $\torus{D}$ and the sphere $\sphere{D}$. The torus $\torus{D}$ can be written as a Cartesian product of $D$ copies of $\sphere{1}$, which will allow us to build flows on $\torus{D}$ autoregressively. The sphere $\sphere{D}$ can be written as a transformation of the cylinder $\sphere{D-1}\times[-1, 1]$, which will allow us to build flows on $\sphere{D}$ recursively, using flows on $\sphere{1}$ and $[-1, 1]$ as building blocks. By combining these constructions, we can build a wide range of compact connected manifolds of interest to fundamental and applied sciences.

\subsection{Flows on the Circle $\sphere{1}$}\label{sec:flows_circle}

Depending on what is most convenient, we will sometimes view $\sphere{1}$ as embedded in $\R^2$, that is $\left\{(x_1, x_2)\in\R^2;x_1^2+x_2^2=1\right\}$, or
parameterize it  by a coordinate $\theta\in[0, 2\pi]$, identifying $0$ and $2\pi$ as the same point.
In this section, we describe how to construct a diffeomorphism $f$ that maps the circle to itself.

Since $0$ and $2\pi$ are identified as the same point, the transformation $f: [0, 2\pi]\rightarrow[0, 2\pi]$ must satisfy appropriate boundary conditions to be a valid diffeomorphism on the circle. The following conditions are sufficient:
\begin{align}
    f(0) &= 0, \label{eq:bc.0}\\
    f(2\pi) &= 2\pi, \label{eq:bc.1}\\
    \nabla f(\theta) & > 0, \label{eq:bc.mono}\\
    \nabla f(\theta)|_{ \theta=0 } &= \nabla f(\theta)|_{ \theta=2\pi }.\label{eq:bc.smooth}
\end{align}
The first two conditions ensure that $0$ and $2\pi$ are mapped to the same point on the circle. The third condition ensures that the transformation is strictly monotonic, and thus invertible. Finally, the fourth condition ensures that the Jacobians agree at $0$ and $2\pi$, thus the probability density is continuous.

A restriction in the above conditions is that $0$ and $2\pi$ are fixed points. Nonetheless, this restriction can be easily overcome by composing a transformation $f$ satisfying these conditions with a phase translation \mbox{$\theta\mapsto (\theta +\phi) \!\!\mod 2\pi$}, where $\phi$ can be a learnable parameter. Such a phase translation is volume-preserving, so it does not incur a volume correction in the calculation of the probability density.

Given a collection $\set{f_i}_{i=1, \ldots, K}$ of transformations satisfying the above conditions, we can combine them into a more complex transformation $f$ that also meets these conditions.
One such mechanism for combining transformations is function composition $f = f_K \circ \cdots \circ f_1$, which can easily be seen to satisfy \Cref{eq:bc.0,eq:bc.mono,eq:bc.1,eq:bc.smooth}. Alternatively, we can combine transformations using convex combinations, as any convex combination defined by
\begin{equation}\label{eq:convex_combination}
    f(\theta) = {\sum}_{i} \rho_i f_i(\theta),\,\text{ where }\rho_i \geq 0 \text{ and } {\sum}_i \rho_i = 1
\end{equation} also satisfies \Cref{eq:bc.0,eq:bc.mono,eq:bc.1,eq:bc.smooth}.
By alternating between function compositions and convex combinations, we can construct expressive flows on $\sphere{1}$ from simple ones.

Next, we describe three circle diffeomorphisms that by construction satisfy the above conditions: \textit{M\"obius transformations}, \textit{circular splines}, and \textit{non-compact projections}.

\subsubsection{M\"obius Transformation}\label{sec:moebius}

M\"obius transformations have previously been used to define distributions on the circle and sphere \citep[see e.g.][]{kato2008circular,wang2013extensions,kato2015m}. We will first describe the M\"obius transformation in the general case of the sphere $\sphere{D}$, and then show how to adapt it, when $D=1$, to create expressive flows on the circle.

\setlength\intextsep{2pt}
\begin{wrapfigure}{L}{82pt}
\begin{tikzpicture}
\draw (0,0) circle (1cm);
\draw (0.5, 0.) circle (0.02cm) node[anchor=south west] {$\omega$};
\draw (0.9285714, -0.37115386) circle (0.02cm) node[anchor=north west] {$z$};
\draw (-0.5, 0.8660254) -- (0.9285714, -0.37115386);
\draw (-0.5, 0.8660254) -- (0.5, -0.8660254);
\draw (0., 0.) circle (0.02cm) node[anchor=north east] {$0$};
\draw (-0.5, 0.8660254) circle (0.02cm) node[anchor=south east] {$g_\omega(z)$};
\draw (0.5, -0.8660254) circle (0.02cm) node[anchor=north west] {$h_\omega(z)$};
\end{tikzpicture}
\end{wrapfigure}
Consider $\sphere{D}$, the unit sphere in $\R^{D+1}$. Let $\omega$ be a point in $\R^{D+1}$ with norm strictly less than $1$. For any point $z$ in $\sphere{D}$, draw a straight line between $z$ and $\omega$ (as shown on the left). This line intersects $\sphere{D}$ at exactly two distinct points. One is $z$. Call the other one $g_\omega(z)$.

\begin{defin}
We define the M\"obius transformation $h_\omega(z)$ of $z$ with centre $\omega$ to be $-g_\omega(z)$.
\end{defin}

An explicit formula for $h_\omega$ is given by
\begin{equation}\label{eq:h_omega}
    h_\omega(z) = \frac{1-\norm{\omega}^2}{\norm{z-\omega}^2}(z-\omega) - \omega.
\end{equation}
When $\omega=0$, the transformation $h_\omega$ is just the identity.
In general, the variables $z$ and $\omega$ in \Cref{eq:h_omega} are meant to be $(D+1)$-dimensional real vectors. When $D=1$, the equation also reads correctly if $z$ and $\omega$ are taken to be complex numbers. In this case, \Cref{eq:h_omega} and the M\"obius transformations of the complex plane are related as explained in formula (6) of \citet{kato2015m}.

The transformation $h_\omega$ expands the part of the sphere that is close to $\omega$, and contracts the rest. Hence, it can transform a uniform base distribution into a unimodal smooth distribution parameterized by $\omega$. One property of the M\"obius transformation is that it does not become more expressive by composing various $h_{\omega}$. This is because the set of transformations $Rh_\omega$, where $R$ can be any matrix in $\text{SO}(D+1)$, forms a group under function composition \citep[see Theorem 2 of][]{kato2015m}. Since composition of two such transformations remains a member of the group, their expressivity is not increased.

In case of the circle (where $D=1$), we would like to get expressive flows using a convex combination as in \Cref{eq:convex_combination} of M\"obius transformations. The transformation $h_\omega$ defines a map from angles to angles that satisfies \Cref{eq:bc.mono,eq:bc.smooth}, but it does not satisfy \Cref{eq:bc.0,eq:bc.1}. This can be fixed by adding a rotation after $h_\omega$. Let $R_{\omega}$ be a rotation in $\R^2$ with centre $(0,0)$ that maps $h_\omega(1, 0)$ back to $(1, 0)$. We define $f_\omega(\theta)$ to be the polar angle, in $[0, 2\pi)$, of $R_{\omega}\circ h_\omega(z)$, and extend $f_\omega$ by continuity to the whole range $[0, 2\pi]$ by setting $f_\omega(2\pi) = 2\pi$. This function satisfies \Cref{eq:bc.0,eq:bc.mono,eq:bc.1,eq:bc.smooth}.
We can then easily combine various $f_{\omega}$ via convex combinations and obtain expressive distributions on $\sphere{1}$---see \Cref{fig:expressive_moebius} in the appendix for an illustration. Unlike a single M\"obius transform, a convex combination of two or more M\"obius transforms is not analytically invertible, but it can be numerically inverted with precision $\epsilon$ using bisection search with $\bigo{\log \frac{1}{\epsilon}}$ iterations.

\subsubsection{Circular Splines (CS)}\label{sec:circular_splines}

Spline flows is a methodology for creating arbitrarily flexible flow transformations from $\R$ to itself, first proposed by \citet{muller2019neural} and further developed by \citet{durkan2019cubic, durkan2019neural}. Here, we will show how to adapt the \textit{rational-quadratic spline flows} of \citet{durkan2019neural} to satisfy the sufficient boundary conditions of circle diffeomorphisms.

Rational-quadratic spline flows define the transformation $f:\R\rightarrow\R$ piecewise as a combination of $K$ segments, with each segment being a simple rational-quadratic function. Specifically, the transformation is parameterized by a set of $K+1$ knot points $\set{x_k, y_k}_{k=0, \ldots, K}$ and a set of $K$ slopes $\set{s_k}_{k=1, \ldots, K}$, such that $x_{k-1} < x_k$, $y_{k-1} < y_k$ and $s_k>0$ for all $k=1,\ldots, K$.
Then, in each interval $[x_{k-1}, x_k]$, the transformation $f$ is defined to be a rational quadratic:
\begin{equation}
    f(\theta) = \frac{\alpha_{k2} \theta^2 + \alpha_{k1} \theta +  \alpha_{k0}}{\beta_{k2} \theta^2 + \beta_{k1} \theta +  \beta_{k0}},
\end{equation}
where the coefficients $\set{\alpha_{ki}, \beta_{ki}}_{i=0, 1, 2}$ are chosen so that $f$ is strictly monotonically increasing and $f(x_{k-1}) = y_{k-1}$, $f(x_{k}) = y_{k}$, and $\nabla f(\theta)|_{\theta=x_k} = s_k$ for all $k=1\ldots, K$ \citep[see][for more details]{durkan2019neural}.

We can easily restrict $f$ to be a diffeomorphism from $[0, 2\pi]$ to itself, by setting $x_0 = y_0 = 0$ and $x_K = y_K = 2\pi$. This construction satisfies the first three sufficient conditions in \Cref{eq:bc.0,eq:bc.1,eq:bc.mono}, and can be used to define probability densities on the closed interval $[0, 2\pi]$. In addition, by setting $s_1=s_K$, we satisfy the fourth condition in \Cref{eq:bc.smooth}, and hence we obtain a valid circle diffeomorphism which we refer to as a \textit{circular spline (CS)}.

Circular splines can be made arbitrarily flexible by increasing the number of segments $K$. Therefore, unlike M\"obius transformations, it is not necessary to combine them via convex combinations to increase their expressivity. An advantage of circular splines is that they can be inverted exactly, by first locating the corresponding segment (which can be done in $\bigo{\log K}$ iterations using binary search since the segments are sorted), and then inverting the corresponding rational quadratic (which can be done analytically by solving a quadratic equation).

\subsubsection{Non-Compact Projection (NCP)}\label{sec:ncp}

As discussed in the introduction, \citet{gemici2016normalizing} project the manifold to $\R^D$, apply the usual flows there, and then project $\R^D$ back to the manifold. Naively applying this method can be numerically unstable. However, here we show that, with some care, the method can be specialized to $\sphere{1}$ in a numerically stable manner. Since this method involves projecting $\sphere{1}$ to the non-compact space $\R$, we refer to it as \textit{non-compact projection (NCP)}\@.

We will use the projection $x: (0, 2\pi) \rightarrow \R$ defined by $x(\theta) = \tan\br{\frac{\theta}{2} - \frac{\pi}{2}}$. This projection maps the circle minus the point $\theta=0$ bijectively onto $\R$. Applying the affine transformation $g(x) = \alpha x + \beta$ in the non-compact space, where $\alpha>0$ and $\beta$ are learnable parameters, defines a corresponding flow on the circle, given by
\begin{equation}\label{eq:ncp}
\begin{aligned}
    f(\theta) &= x^{-1}\circ g\circ x(\theta)\\ &= 2\tan^{-1}\br{\alpha \tan\br{\frac{\theta}{2} - \frac{\pi}{2}} + \beta} + \pi,
\end{aligned}
\end{equation}
with gradient
\begin{equation*}
    \nabla f(\theta) = \left[ \frac{1 + \beta^2}{\alpha} \sin^2\left(\frac{\theta}{2}\right) + \alpha \cos^2\left(\frac{\theta}{2}\right) - \beta \sin\theta  \right]^{-1}.
\end{equation*}
Even though the expression for $f$ is not defined at the endpoints $0$ and $2\pi$, the expression for the gradient is. The transformation satisfies the appropriate boundary conditions,
\begin{equation*}
\begin{aligned}
    \lim_{\theta \rightarrow 0^+} f(\theta) &= 0, \\
    \lim_{\theta \rightarrow 2\pi^-} f(\theta) &= 2\pi, \\
    \nabla f(\theta)|_{\theta=0} = \nabla f(\theta)|_{\theta=2\pi} &= \alpha^{-1}.
\end{aligned}
\end{equation*}
Therefore, we can extend $f$ to $[0, 2\pi]$ by continuity such that $f(0)=0$ and $f(2\pi)=2\pi$, which yields a valid circle diffeomorphism. Although not immediately obvious, NCP flows and M\"obius transformations are intimately related, as explained in \Cref{sec:ncp_mobius_equiv} \citep[see also][Section 2.1]{downs2002circular}.

The above boundary conditions are satisfied when the transformation $g$ is affine, but they are not generally satisfied when $g$ is an arbitrary diffeomorphism. This limits the type of flow we can put on the non-compact space. Therefore, instead of making $g$ more expressive, we choose to increase the expressivity of the NCP flow by combining multiple transformations $f$ via convex combinations.

Similarly to M\"obius transformations, NCP flows form a group. That's because affine maps with positive slope on $\R$ form a group, and if $g_1$ and $g_2$ are affine maps and $f_k=x\circ g_k\circ x^{-1}$, then $f_1\circ f_2 = x\circ g_1\circ g_2\circ x^{-1}$. Hence, the composition of two NCPs with parameters $(\alpha_k, \beta_k),k=1,2$ is another NCP with parameters $(\alpha_1 \alpha_2, \beta_1 + \alpha_1 \beta_2)$. So composing two NCP flows would produce a new NCP flow, leading to no increase in expressivity.

A potential issue with NCP is that, near the endpoints $0$ and $2\pi$, evaluating $f$ using \Cref{eq:ncp} directly is numerically unstable. To circumvent this numerical difficulty, we can use equivalent linearized expressions when $\theta$ is near the endpoints. For example, for $\theta$ close to $0$ we can approximate $f(\theta) \approx \theta/\alpha$, whereas for $\theta$ close to $2\pi$ we have $f(\theta) \approx 2\pi + (\theta - 2\pi)/\alpha$.

\subsection{Generalization to the Torus $\torus{D}$}\label{sec:ar:torus}

Having defined flows on the circle $\sphere{1}$, we can easily construct autoregressive flows on the $D$-dimensional torus $\torus{D}\cong (\sphere{1})^D$.
Any density $p(\theta_1, \ldots, \theta_D)$ on $\torus{D}$ can be decomposed via the chain rule of probability as
\begin{align}
    p(\theta_1, \ldots, \theta_D) = {\prod}_{i} p(\theta_i\g\theta_1, \ldots, \theta_{i-1}),
\end{align}
where each conditional $p(\theta_i\g\theta_1, \ldots, \theta_{i-1})$ is a density on $\sphere{1}$.
Each conditional density can be implemented via a flow $f_{\psi_i}: \sphere{1}\rightarrow\sphere{1}$, whose parameters $\psi_i$ are a function of $(\theta_1, \ldots, \theta_{i-1})$.
Thus the joint transformation $f: \torus{D}\rightarrow\torus{D}$ given by $f(\theta_1, \ldots, \theta_D) = (f_{\psi_1}(\theta_1), \ldots, f_{\psi_D}(\theta_D))$ is an \textit{autoregressive flow} on the torus. In the terminology of \citet[Section 3.1]{papamakarios2019normalizing}, an autoregressive flow on a torus is simply an autoregressive flow whose \textit{transformers} are circle diffeomorphisms, i.e.~obey the boundary conditions in \Cref{eq:bc.0,eq:bc.1,eq:bc.mono,eq:bc.smooth}. As with any autoregressive flow, the Jacobian of $f$ is triangular, therefore the Jacobian determinant in the density calculation in \Cref{eq:flow_jac} can be computed efficiently as the product of the diagonal terms.

The parameters $\psi_i$ of the $i$-th transformer are a function of $(\theta_1, \ldots, \theta_{i-1})$ known as the $i$-th \textit{conditioner}. In order to guarantee that the conditioners are periodic functions of each $\theta_i$, we can make $\psi_i$ be a function of $(\cos{\theta_1}, \sin{\theta_1}, \ldots,\cos{\theta_{i-1}}, \sin{\theta_{i-1}})$ instead. In our experiments, we implemented the conditioners using \textit{coupling layers} \citep{dinh2017density}. Implementations based on \textit{masking} \citep{kingma2016iaf, papamakarios2017masked} are also possible.

More generally, autoregressive flows can be applied in the same way on any manifold that can be written as a Cartesian product of circles and intervals, such as the $2$-dimensional cylinder. Flows on intervals can be constructed e.g.~using regular (non-circular) splines as described in \Cref{sec:circular_splines}. Thus, by taking $f_{\psi_i}$ to be either a circle diffeomorphism or an interval diffeomorphism as required, we can handle arbitrary products of circles and intervals.

\subsection{Generalization to the Sphere $\sphere{D}$}\label{sec:ar:sphere}

The M\"obius transformation (\Cref{sec:moebius}) can in principle be used to define flows on the sphere $\sphere{D}$ for $D\ge 2$, but, as noted, its expressivity does not increase by composition.
Increasing the expressivity via convex combinations
in $D=1$ was only possible because we expressed diffeomorphisms on $\sphere{1}$ as strictly increasing functions on $[0,2\pi]$. This construction however does not readily extend to $D\geq 2$.
Instead, we propose two alternative flow constructions for $\sphere{D}$: a \textit{recursive construction} that uses cylindrical coordinates, and a construction based on the \textit{exponential map}. In what follows,  $\sphere{D}$ will be embedded in $\R^{D+1}$ as
$\left\{(x_1,\ldots,x_{D+1})\in\R^{D+1};\sum_ix_i^2=1\right\}$.

\subsubsection{Recursive Construction}\label{sec:ar:sphere:recursion}

In \Cref{sec:flows_circle}, we described three methods for building univariate flows on $\sphere{1}$. We also described how to build univariate flows on the interval $[-1, 1]$ using splines (\Cref{sec:circular_splines}). Using circle and interval flows as building blocks, we will construct a multivariate flow on $\sphere{D}$ for $D\ge 2$ by recursing over the dimension of the sphere.

\begin{figure}[t!]
\hspace{-7pt}
    \begin{subfigure}[t]{0.5\textwidth}
        \centering
        \includegraphics[width=\textwidth]{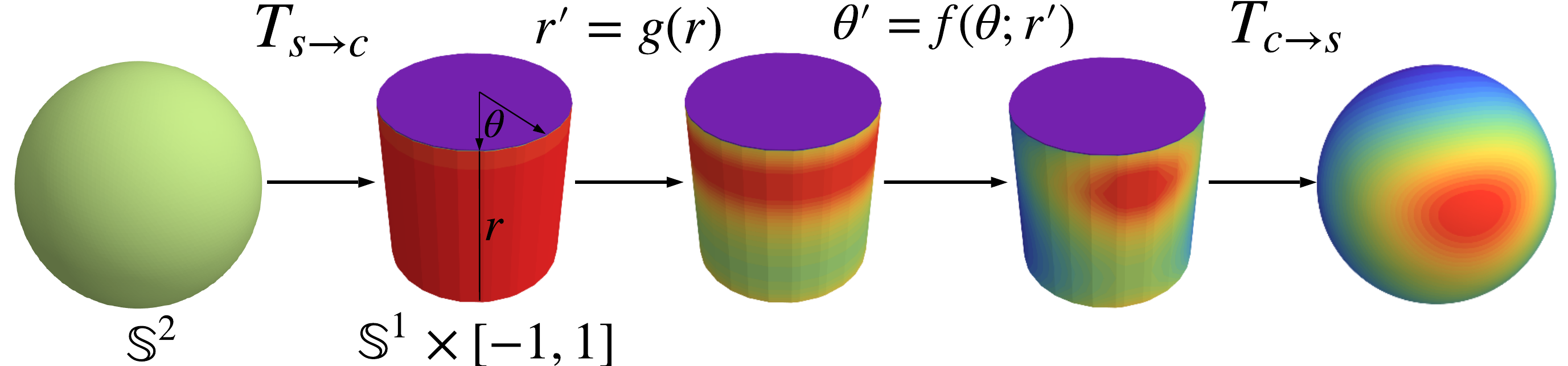}
    \end{subfigure}
\caption{Illustration of the recursive flow on the sphere $\sphere{2}$.}\label{fig:flow:s2:diagram}
\end{figure}

Our construction works as follows. First, we transform the sphere $\sphere{D}$ into the cylinder $\sphere{D-1}\times[-1, 1]$ using the map
\begin{equation}\label{eq:flow_sd_1}
\Tsc(x) = \left(\frac{x_{1:D}}{\sqrt{1 - x_{D+1}^2}}, \,x_{D+1}\right).
\end{equation}
Then a two-stage autoregressive flow is applied on the cylinder. The `height' $r\in[-1, 1]$ is transformed first by a univariate flow $g: [-1, 1] \rightarrow [-1, 1]$, and the `base' $z$ is transformed second by a conditional flow $f: \sphere{D-1} \rightarrow \sphere{D-1}$ whose parameters are a function of $g(r)$, that is
\begin{equation}\label{eq:flow_sd_2}
T_{c \rightarrow c}(z,r) = \Big(f(z; g(r)), \,g(r)\Big).
\end{equation}
Finally, the cylinder is transformed back to the sphere by
\begin{equation}\label{eq:flow_sd_3}
\Tcs(z, r) = \br{z \sqrt{1 - r^2}, \,r}.
\end{equation}
The flow $h: \sphere{D} \rightarrow \sphere{D}$ is defined to be the composition $h = \Tcs \circ T_{c \rightarrow c} \circ \Tsc$. When stacking multiple $h$ (each with its own parameters), all the internal terms in $\Tsc\circ\Tcs$ cancel out, so it is more efficient to just not compute them. Only the first $\Tcs$ and last $\Tsc$ need to be computed. \Cref{fig:flow:s2:diagram} illustrates this procedure on $\sphere{2}$.

The above recursion continues until we reach $\sphere{1}$, on which we can use the flows we have already described. Unrolling the recursion up to $\sphere{1}$ is equivalent to first transforming $\sphere{D}$ into $\sphere{1}\times[-1, 1]^{D-1}$, then applying an autoregressive flow on $\sphere{1}\times[-1, 1]^{D-1}$ as described in \Cref{sec:ar:torus}, and finally transforming $\sphere{1}\times[-1, 1]^{D-1}$ back to $\sphere{D}$. In practice, we can use any ordering of the variables in the autoregressive flow and not just the one implied by the recursion, and we can compose multiple autoregressive flows before transforming back to $\sphere{D}$. \Cref{fig:flow:s2:diagram:detailed} in the appendix illustrates the model with the recursion fully unrolled.

At this point, it is important to examine carefully the transformations $\Tsc$ and $\Tcs$. Maps from $\sphere{D-1}\times [-1, 1]$ to $\sphere{D}$ can never be diffeomorphisms since the topology of these two spaces differ.
Nevertheless, $\Tsc$ and $\Tcs$ are invertible almost everywhere, in the sense that we can remove sets of measure $0$ from $\sphere{D-1}\times[-1, 1]$ and $\sphere{D}$, and the restriction of these maps to these subsets are diffeomorphisms.
However, extra care must be taken to check that updates to the density caused by $\Tsc$ and $\Tcs$ are still numerically stable, and that the density is well-behaved everywhere on $\sphere{D}$.
In \Cref{sec:tcs}, we prove that the update to a base density $\pi(z,r)$ due to $\Tcs$ is
\begin{equation}\label{eq:tcs_density_update}
    p(\Tcs(z, r)) = \frac{\pi(z, r)}{(1-r^2)^{\frac{D}{2}-1}}.
\end{equation}
The update due to $\Tsc$ is simply the inverse of the above. Additionally, in \Cref{sec:tcs} we prove that the combined density update due to $h = \Tcs \circ T_{c \rightarrow c} \circ \Tsc$ does not have any singularity whenever $g$ is such that $g'(-1)>0$ and $g'(1)>0$. Since we implement $g$ with spline flows (\Cref{sec:circular_splines}), this condition is satisfied by construction, so the flow density is guaranteed to be finite.

Since the density update due to $\Tsc$ and $\Tcs$ can be done analytically and the rest of the transformation is autoregressive, the overall density update is efficient to compute. For $D=2$, the volume correction in \Cref{eq:tcs_density_update} is equal to $1$, and therefore $\Tcs$ and $\Tsc$ preserve infinitesimal volume elements.
Finally, we observe that the well-known \textit{von Mises--Fisher} distribution on $\sphere{2}$  \citep{fisher1953dispersion} can be obtained by transforming a uniform base density with \Cref{eq:flow_sd_1,eq:flow_sd_2,eq:flow_sd_3}, where $f$ is the identity and $g$ has a simple functional form \citep[for details, see][]{jakob2012numerically}.

\subsubsection{Exponential-Map Flows}\label{sec:exp_flow}

Ideally, we would like to specify flows directly on the manifold and avoid mapping between non-diffeomorphic sets. This motivates us to explore \textit{exponential-map flows}, a
mechanism for building flows on spheres
proposed by \citet{sei2013jacobian}.

The main idea consists of two steps. First, we define a \textit{cost-convex} scalar field $\phi(x)$ on $\sphere{D}$---for the definition of `cost-convex', see \citep[Formula 1]{sei2013jacobian}. 
Second, we construct a flow using the exponential map of the gradient $\nabla \phi(x) \in T_x\sphere{D}$, where $T_x\sphere{D}$
is the tangent space of $\sphere{D}$ at $x$. The exponential map $\exp_x: T_x\mathcal{M} \rightarrow \mathcal{M}$ on a Riemannian manifold $\mathcal{M}$
is defined as the terminal point $\gamma(1)$ of a geodesic $\gamma: [0,1] \rightarrow \mathcal{M}$ that passes through $\gamma(0)=x$ with velocity $\dot{\gamma}(0) = v$ \citep{kobayashi1963foundations}.
For a general manifold it is hard to compute exponential maps, but for the sphere $\sphere{D}$ the exponential map is straightforward:
\begin{align}
    \exp_x(v) &= x\cos \|v\| + \frac{v}{\|v\|}\sin \|v\|,
\end{align}
where $\|\cdot\|$ is the Euclidean norm.

Parameterizing a general cost-convex scalar field $\phi(x)$ remains non-trivial on $\sphere{D}$. To solve this, \citet{sei2013jacobian} proposes to construct $\phi(x)$ via a convex combination of simple scalar functions $\phi_i: [0, \pi] \rightarrow \mathbb{R}$ that satisfy $\phi_i'(0)=\phi_i'(\pi)=0$ and $\phi_i'' > -1$. Specifically, $\phi(x)$ is given by
\begin{equation}
    \phi(x) = {\sum}_i \alpha_i \phi_i(d(x, \mu_i)),
\end{equation}
where $d(x, \mu_i)$ is the geodesic distance between $x\in\sphere{D}$ and $\mu_i \in \sphere{D}$, $\alpha_i \geq 0$ and $\sum_i \alpha_i \le 1$. The exponential map of the gradient field of $\phi(x)$ is guaranteed to be a diffeomorphism from $\sphere{D}$ to $\sphere{D}$ \citep[Lemma 4]{sei2013jacobian}.

We found that it is challenging to learn concentrated multi-modal densities on $\sphere{D}$ using the polynomial and high-frequency scalar fields proposed by \citet{sei2013jacobian}.
To address this limitation, we introduce a scalar field which is inspired by radial flows \cite{tabak2013family, rezende2015variational}  and is given by
\begin{align}\label{eq:sum_of_radial_scalar_field}
    \phi(x) &= {\sum}_{i} \frac{\alpha_i}{\beta_i} e^{\beta_i (x^T\mu_i - 1)},
\end{align}
where $\alpha_i \ge 0$, $\sum_i \alpha_i \leq 1$, $\mu_i \in \sphere{D}$ and $\beta_i > 0$.
It is straightforward to see that the functions
\begin{equation*}
\phi_i(d(x, \mu_i)) = \frac{1}{\beta_i}e^{\beta_i (\cos d(x, \mu_i) - 1)} = \frac{1}{\beta_i}e^{\beta_i (x^T\mu_i - 1)}
\end{equation*}
satisfy the required conditions, therefore the map
\begin{equation}\label{eq:exp_flow}
\begin{split}
    f(x) &= \exp_x \nabla \phi(x) \\
         &= x \cos \| \nabla \phi(x)\| + \frac{\nabla \phi(x)}{\| \nabla \phi(x) \|} \sin \| \nabla \phi(x)\|
\end{split}
\end{equation}
is a diffeomorphism from $\sphere{D}$ to $\sphere{D}$.

As shown in \Cref{sec:volume.elements}, the density update due to $f$ is
\begin{equation}
    p(f(x)) = \frac{\pi(x)}{\sqrt{\det\br{ E(x)^\top J(x)^\top J(x) E(x)}}},
\end{equation}
where $J(x)$ is the Jacobian of $f$ at $x$ when $f$ is seen as a map from $\R^{D+1}$ to $\R^{D+1}$, and the columns of $E(x)$ form an orthonormal basis of vectors tangent to the sphere at $x$. Unlike the recursive flow which admits an efficient density update, computing the above density update costs $\bigo{D^3}$, so the exponential-map flow is only practical for small $D$.

\section{Related Work}
\label{sec:related_work}
The study of distributions on objects such as angles, axes and directions has a long history, and is known as \textit{directional statistics} \citep{mardia2009directional}.
According to the taxonomy of \citet{navarro2017multivariate}, directional statistics traditionally uses three approaches for defining distributions on tori and spheres: \textit{wrapping}, \textit{projecting}, and \textit{conditioning}.

\textit{Wrapping} is typically used to define distributions on $\sphere{1}$. The idea is to start with a distribution $p(x)$ on $\R$, and wrap it around $\sphere{1}$ by writing $x=\theta+2\pi k$ and marginalizing out $k$.
A common example is the \textit{wrapped Gaussian} \citep[see e.g.][Section 2.1]{jona2012spatial}. A disadvantage of wrapped distributions is that, unless $p(x)$ has bounded support, they require an infinite sum over $k$, which is generally not analytically tractable.

\textit{Projecting} is an alternative approach where the manifold is embedded into $\R^M$ with $M>D$, a distribution $p(x)$ is defined on $\R^M$, and the probability mass is projected onto the manifold. For example, if the manifold is the sphere $\sphere{D}$ embedded in $\R^{D+1}$, the probability mass in $\R^{D+1}$ can be projected onto $\sphere{D}$ by writing $x=r u_x$ where $r = \norm{x}$ and $u_x$ is a unit vector in the direction of $x$, and then integrating out $r$.
A common example is the projected Gaussian on $\sphere{1}$ \citep[see e.g.][]{wang2013projectednormal}. A difficulty with this approach is that it requires computing an integral over $r$, which may not be generally tractable.

\textit{Conditioning} also begins by embedding the manifold into $\R^M$ and defining a distribution $p(x)$ on $\R^M$ (which here can be unnormalized). Then, the distribution on the manifold $\mathcal{M}$ is obtained by conditioning $p(x)$ on $x\in\mathcal{M}$, that is, restricting $p(x)$ on $\mathcal{M}$ and re-normalizing.
Common examples when $\mathcal{M}=\sphere{1}$ are the \textit{von Mises} \citep[Section 1.3.3]{wang2013extensions} and \textit{generalized von Mises} distributions \citep{gatto2007generalized}, which are obtained by taking an isotropic or arbitrary Gaussian (respectively) in $\R^2$ and conditioning on $\norm{x}=1$. For the case of $\mathcal{M}=\sphere{D}$, a common example is the \textit{von Mises--Fisher} distribution \citep{fisher1953dispersion}, which is obtained by conditioning an isotropic Gaussian in $\R^D$ on $\norm{x}=1$, and reduces to the von Mises for $D=1$. The von Mises--Fisher distribution can also be extended to the \textit{Stiefel manifold} of sets of $N$ orthonormal $D$-dimensional vectors \citep{downs1972orientation},
%\footnote{The Stiefel manifold is isomorphic to $\sphere{D}$ for $N=1$, to $\mathrm{SO}(D)$ for $N=D-1$ and to $\mathrm{O}(D)$ for $N=D$.}
and it can be generalized to the more flexible \textit{Kent} distribution, also known as \textit{Fisher--Bingham}, on $\sphere{D}$ \citep{kent1982fisherbingham}. 
Finally, distributions on $\torus{D}$ can be obtained by defining $p(x)$ in $\R^{2D}$ and conditioning on $x_{i}^2 + x_{D+i}^2 = 1$ for all $i=1, \ldots, D$. Examples include the \textit{multivariate von Mises} \citep{mardia2008multivariate} and \textit{multivariate generalized von Mises} distributions \citep{navarro2017multivariate}, which correspond to $p(x)$ being Gaussian with constrained or arbitrary covariance matrix respectively. A drawback of conditioning is that computing the normalizing constant of the resulting distribution requires integrating $p(x)$ on $\mathcal{M}$, which is not generally tractable.

In general, the above three strategies lead to distributions with tractable density evaluation and sampling algorithms only in special cases, which typically yield simple distributions with limited flexibility. One approach for increasing the flexibility of such simple distributions is via combining them into \textit{mixtures} \citep[see e.g.][]{peel2001mixtures, mardia2007protein}. Such mixtures could be used as base distributions for the flows on tori and spheres that we present in this work. However, unlike flows whose expressivity increases via composition, mixtures generally require a large number of components to represent sharp and complex distributions, and can be harder to fit in practice.

A general method for defining distributions on \textit{Lie groups}, which are groups with manifold structure, was proposed by \citet{falorsi2019reparameterizing}. The tangent space of a $D$-dimensional Lie group at the identity element is isomorphic to $\R^D$ and is known as the \textit{Lie algebra} of the group. The method is to define a distribution $p(x)$ on the Lie algebra (e.g.~using standard Euclidean flows), and then map the Lie algebra onto the group using the exponential map. If the Lie group is a compact connected manifold, the exponential map is surjective but not injective, and the method can be thought of as a generalization of \textit{wrapping}. When the Lie group is $\text{U}(1) \cong \sphere{1}$, we recover wrapped distributions on $\sphere{1}$, since in this case the exponential map is $x \mapsto (\cos x, \sin x)$. When the Lie group is $\torus{D}$, the exponential map wraps $\R^D$ around the torus along each dimension. This involves an infinite sum, but even if the sum is restricted to be finite (e.g.~by bounding the support of $p(x)$), the number of terms to be summed scales exponentially with $D$. In addition to the above, \citet{falorsi2019reparameterizing} provide concrete examples for $\mathrm{SO}(3)$ and $\mathrm{SE}(3)$. The main drawbacks of this method is that the exponential map cannot be easily composed with other maps (since it is not generally bijective), and can be expensive to compute in high-dimensions.

Finally, a general method for defining normalizing flows on a Riemannian manifold $\mathcal{M}$ was proposed by \citet{gemici2016normalizing}. This method relies on the existence of two maps: an embedding $g:\mathcal{M}\rightarrow\R^M$ and a coordinate chart $\phi:\mathcal{M}\rightarrow\R^D$, where $D\le M$. The idea is to apply the usual Euclidean flows on $\R^D$, and then transform the density onto the (embedded) manifold via the map $g\circ \phi^{-1}$.
The density update associated with this transformation is $\sqrt{\det J^\top J}$, where $J$ is the Jacobian of $g\circ \phi^{-1}$ (as shown in  \Cref{sec:volume.elements}).
This method leads to composable transformations, but is better suited for manifolds that are homeomorphic to $\R^D$ (so that $\phi$ is well-defined).
Compact manifolds such as tori and spheres are not homeomorphic to $\R^D$, hence $\phi$ cannot be defined everywhere on $\mathcal{M}$.
In such cases, the transformed density may not be well-behaved, which can create numerical instabilities during training.

\section{Experiments}
\label{sec:experiments}
For M\"obius transforms, we reparameterized the centre $\omega$, where $\norm{\omega} < 1$, in terms of an unconstrained parameter $\omega'\in\R^2$ via $\omega = \frac{0.99}{1 + \norm{\omega'}}\omega'$. The constant $0.99$ is used to prevent $\norm{\omega}$ from being too close to $1$.
For circular splines, we prevent slopes from becoming smaller than $10^{-3}$ by adding $10^{-3}$ to the output of a softplus of unconstrained parameters.
For autoregressive flows, the parameters of each conditional transformation are computed by a ReLU-MLP with layer sizes $[N_i, 64, 64, N_o]$, where $N_i$ is the number of inputs and $N_o$ the number of parameters for each transform. All flow models use uniform base distributions.

We evaluate and compare our proposed flows based on their ability to learn sharp, correlated and multi-modal target densities. 
We used targets $p(x) = e^{-\beta u(x)}/Z$ with inverse temperature $\beta$ and normalizer $Z$. We varied $\beta$ to create targets with different degrees of concentration/sharpness.
The models $q_{\eta}$ were trained by minimizing the KL divergence
\begin{equation}
    \mathrm{KL}(q_{\eta}\,\|\,p) = \mathbb{E}_{q_{\eta}(x)}\!\left[\ln q_{\eta}(x) + \beta u(x)\right] + \ln Z.
\end{equation}
For an additional evaluation of how well the flows match the target, we used $S=20{,}000$ samples $x_s$ from the \textit{trained} flows
to estimate the normalizer via importance sampling:
\begin{equation}
    Z = \mathbb{E}_{q_{\eta}(x)}\left[\frac{e^{-\beta u(x)}}{q_{\eta}(x)}\right] \approx \frac{1}{S}\sum_{s=1}^S w_s,
\end{equation}
where $w_s = e^{-\beta u(x_s)}/q_{\eta}(x_s)$. The effective sample size, \citep{arnaud2001sequential, liu2008monte}, can be estimated by
\begin{equation}
    \text{ESS} = \frac{\text{Var}_{\text{unif}}\left[e^{-\beta u(x)}\right]}{\text{Var}_q\left[\frac{e^{-\beta u(x)}}{q_{\eta}(x)}\right]} \approx \frac{\left(\sum_{s=1}^S w_s\right)^2}{\sum_{s=1}^S w_s^2}.
\end{equation}
Higher ESS indicates that the flow matches the target better (when reliably estimated). We report ESS as a percentage of the actual sample size.

\begin{figure}[tb]
\hspace{-4pt}\includegraphics[width=1.025\columnwidth]{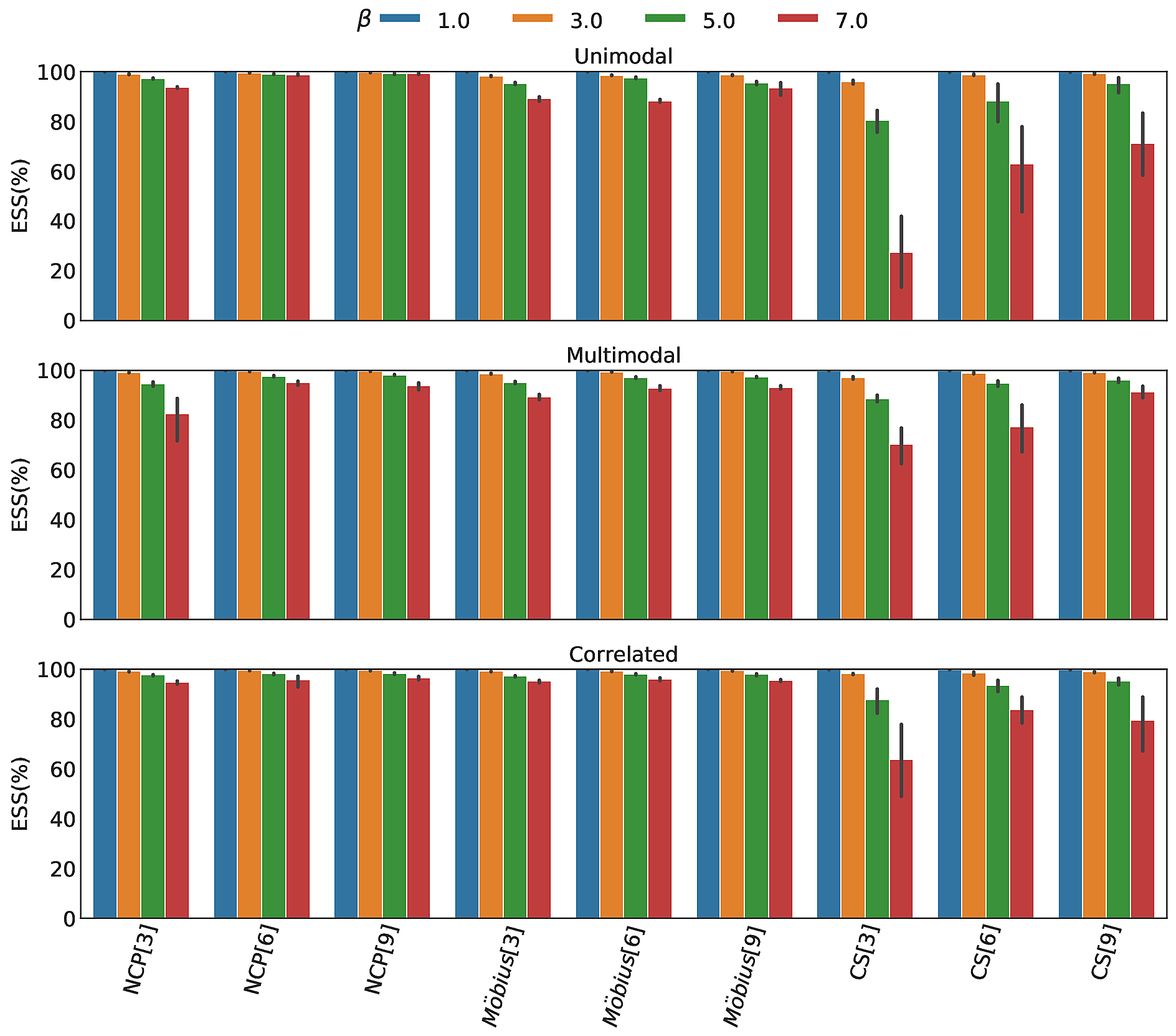}
\caption{Comparing flows on $\mathbb{T}^2$ using ESS, for the targets in \Cref{fig:flow_torus_densities} and various inverse temperatures $\beta$. A larger value of $\beta$ makes the target densities more concentrated and therefore harder to learn.
Numbers in square brackets indicate the number of components $K$ used for each transform. All values are averages across $10$ runs with different neural-network initialization.
See also \Cref{fig:results_torus_all} (appendix) for KL values and further comparisons.
{\bf Top row}: Unimodal target. {\bf Middle row}:  Multi-modal target with $3$ mixture components.  
{\bf Bottom row}: Correlated target.
\label{fig.torus}
}
\end{figure}

\begin{figure}[htb]
\hspace{-7pt}
    \begin{subfigure}[t]{0.5\textwidth}
        \centering
        \includegraphics[width=\textwidth]{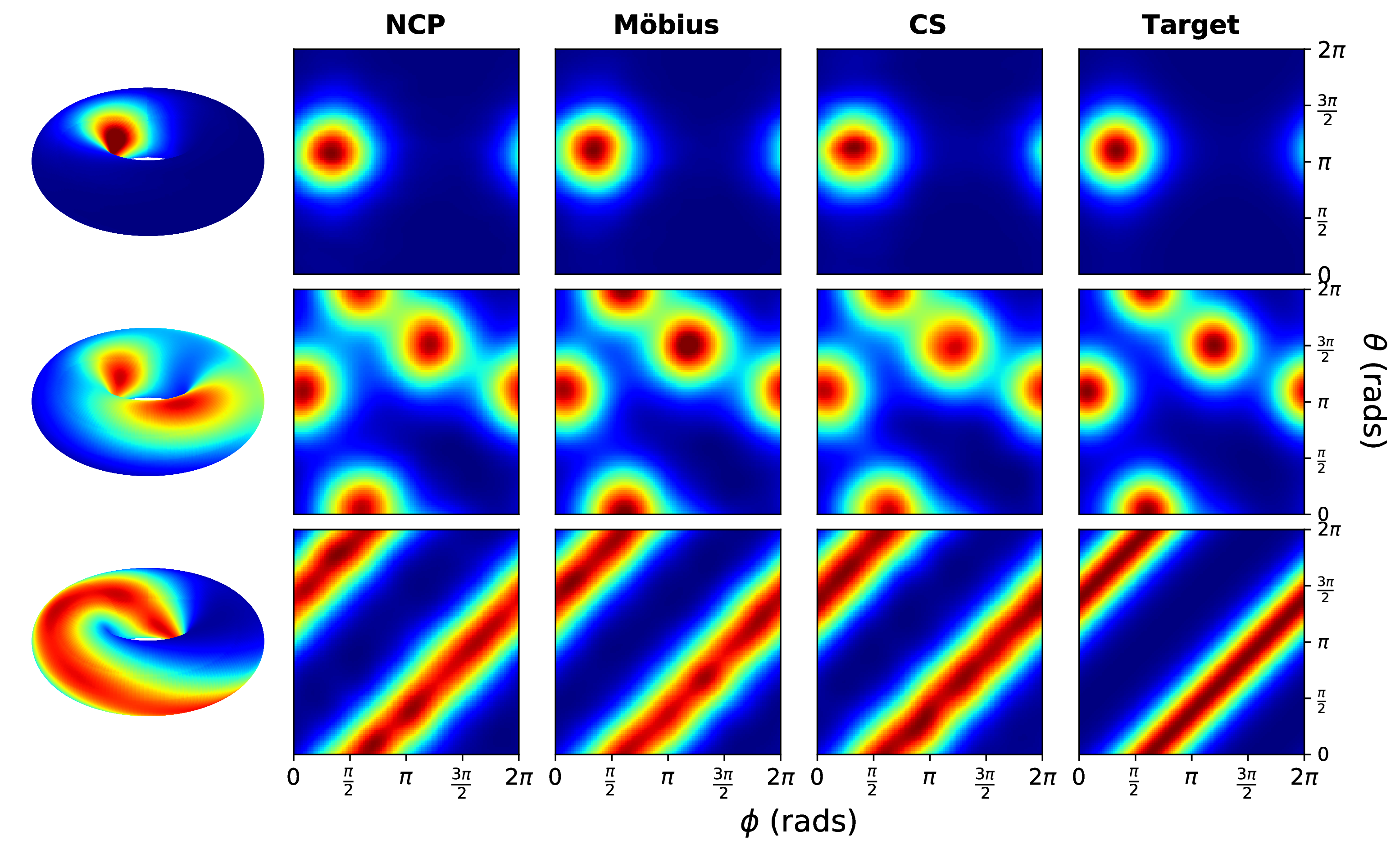}
    \end{subfigure}
\caption{Learned densities on $\torus{2}$ using NCP, M\"obius and CS flows. Densities shown on the torus are from NCP\@. }\label{fig:flow_torus_densities}
\end{figure}

\begin{table}[htb]
\begin{center}
\begin{tabularx}{\columnwidth}{@{}Xcc@{}}
\toprule
\multicolumn{1}{c}{\textbf{Model}} & \multicolumn{1}{c}{\textbf{KL [nats]}} & \multicolumn{1}{c}{\textbf{ESS}} \\ \midrule
\makecell[l]{MS ($N_T=1, K_m=12, K_s=32$)} & $0.05\,(0.01)$ & $90\%$ \\
\makecell[l]{EMP ($N_T=1$)} & $0.50\,(0.09)$ & $43\%$ \\
\makecell[l]{EMSRE ($N_T=1, K=12$)} & $0.82\,(0.30)$ & $42\%$ \\
\makecell[l]{EMSRE ($N_T=6, K=5$)} & $0.19\,(0.05)$ & $75\%$ \\ 
\makecell[l]{EMSRE ($N_T=24, K=1$)} & $0.10\,(0.10)$ & $85\%$ \\\bottomrule
\end{tabularx}
\caption{Comparing baseline and proposed flows on $\mathbb{S}^2$ using KL and ESS\@. The target density is the mixture of $4$ modes shown in \Cref{fig:sphere}\label{table:s2}. We compare recursive M\"obius-spline flow (MS), exponential-map polynomial flow (EMP) and exponential-map sum-of-radial flow (EMSRE)\@. Brackets show error bars on the KL from $3$ replicas of each experiment.
$N_T$ is the number of stacked transformations for each flow; $K_m$ is the number of centres used in M\"obius; $K_s$ is the number of segments in the spline flow; $K$ is the number of radial components in the radial exponential-map flow. The polynomial scalar field is shown in \Cref{sec:polexpmap}.}
\end{center}
\end{table}

\begin{figure}[htb]
    \centering
    \includegraphics[width=0.35\textwidth]{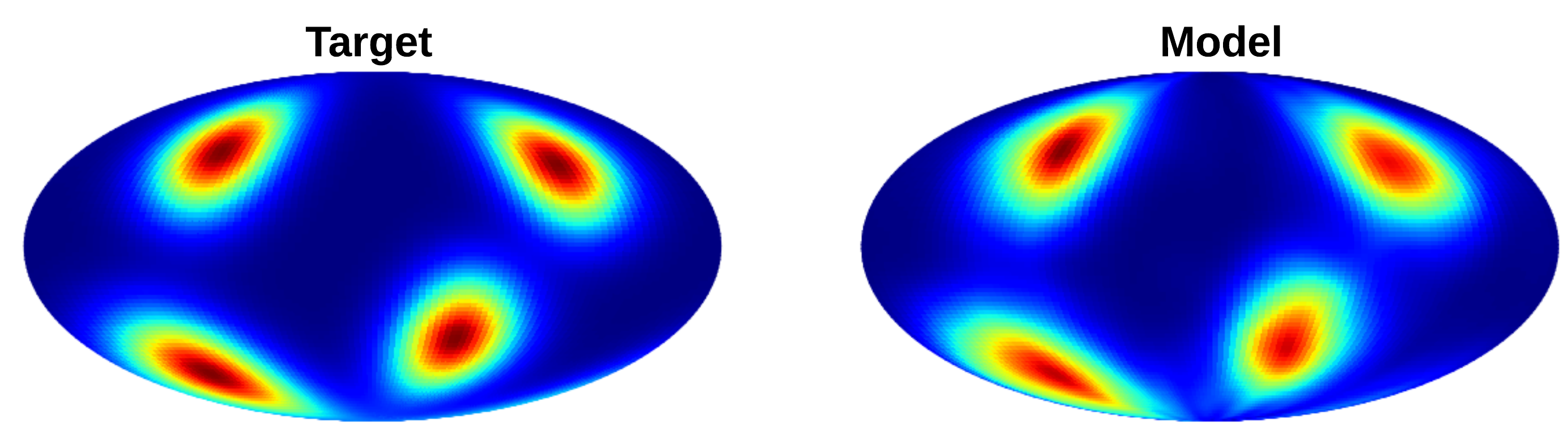}
\caption{Learned multi-modal density on $\sphere{2}$ using exponential-map flows, using the Mollweide projection for visualization. The model is a composition of $24$ exponential-map transforms, using the radial scalar field with $1$ component.
\label{fig:sphere}}
\end{figure}

\begin{figure}[htb]
\centering
\includegraphics[width=\columnwidth]{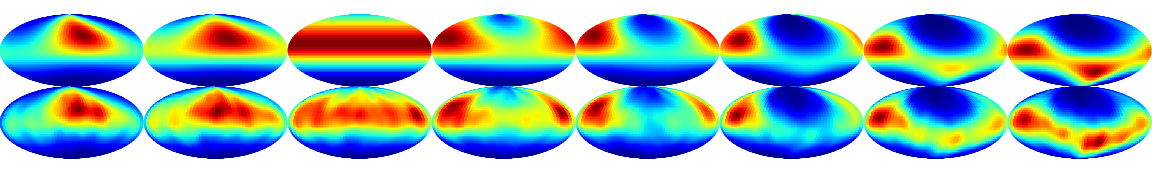}
\caption{Learned multi-modal density on $\SU{2} \cong \sphere{3}$ using the recursive flow. Each column shows an $\sphere{2}$ slice of the $\sphere{3}$ density along a fixed axis using the Mollweide projection. \textbf{Top row}: target density. \textbf{Bottom row}: learned density.
We used a M\"obius transform with $K_m=32$ for the circle, and spline transforms with $K_s = 64$ for the two intervals ($\text{ESS}=84\%$, $\text{KL}=0.14$).\label{fig:su2}}
\end{figure}

All models were optimized using Adam \cite{kingma2015adam} with learning rate $2\times 10^{-4}$, $20{,}000$ iterations, and batch size $256$. The reported error bars are the standard deviation of the average, computed from independent replicas of each experiment with identical hyper-parameters, but with different initialization of the neural network weights. All shown model densities are kernel density estimates using $20{,}000$ samples and Gaussian kernel with bandwidth $0.2$.

\Cref{fig.torus,fig:flow_torus_densities} show results for $\torus{2}$. The target densities are: a unimodal von Mises, a multi-modal mixture of von Mises, and a correlated von Mises (their precise definitions are in \Cref{table.tori.target} of the appendix). 
The results demonstrate that we can reliably learn multi-modal and correlated densities with different degrees of concentration. Among the circle flows, NCP and M\"obius performed the best, whereas CS performed less well for high $\beta$. 
An additional experiment is shown in \Cref{sec:robot_arm}, where flows on $\torus{6}$ are used to learn the posterior density over joint angles of a robot arm.

\Cref{table:s2} shows results for $\mathbb{S}^2$, where we compare the recursive flow to the exponential-map flow with polynomial and radial scalar fields. The recursive flow uses M\"obius for the circle and splines for the interval. For the exponential-map flow, the polynomial scalar field was proposed by \citet{sei2013jacobian} and is shown in \Cref{sec:polexpmap}, whereas the radial scalar field is the new one we propose in \Cref{eq:sum_of_radial_scalar_field}. The target density is the mixture of $4$ modes shown in \Cref{fig:sphere}. We demonstrate substantial improvement relative to the polynomial scalar field, but we find that the exponential-map flow with either scalar field is not yet competitive with the recursive flow.
\Cref{fig:su2} shows an example of learning a density on $\SU{2} \cong \sphere{3}$ using the recursive flow.
Finally, \Cref{sec:globe} shows an example of training a recursive flow (using splines for both the circle and the interval) on data sampled form a `map of the world' density on $\sphere{2}$.

\section{Discussion}
\label{sec:discussion}
This work shows how to construct flexible normalizing flows on tori and spheres of any dimension in a numerically stable manner. Unlike many of the distributions traditionally used in directional statistics, the proposed flows can be made arbitrarily flexible, but have tractable and exact density evaluation and sampling algorithms. We conclude with a comparison of the proposed models, a discussion of their limitations, and some preliminary thoughts on how to extend flows to other manifolds of interest to fundamental physics.

\subsection{Comparison, Scope and Limitations}

Among the flows on the circle, M\"obius and NCP performed the best, with CS performing less well for highly concentrated target densities. However, increasing the expressivity of M\"obius and NCP required convex combinations, whereas CS can be made more expressive by adding more spline segments.
As a result, CS is the cheapest to invert (it can be done analytically), whereas M\"obius and NCP (with more than one component) require a root-finding algorithm such as bisection search. Therefore, in practice it may be preferable to use CS if both density evaluation and sampling are required, and use M\"obius or NCP otherwise.

On $\sphere{D}$, the recursive flow performed better than the exponential-map flow. In addition, the recursive flow scales better to high dimensions since its density can be computed efficiently, whereas the density of the exponential-map flow has a computational cost of $\bigo{D^3}$. The theoretical advantage of the exponential-map flow is that it is intrinsic to the sphere, but this advantage did not result in a practical benefit in our experiments.

\subsection{Towards Normalizing Flows on $\SU{D}$ and $\text{U}(D)$}
\label{sec:future}

The unitary Lie groups $\text{U}(1)$, $\SU{2}$ and $\SU{3}$ are of particular interest to fundamental physics, because the symmetry groups of particle interactions are constructed from them \citep{woit2017quantum}.
We have shown that we can construct expressive flows on
$\text{U}(1)\cong\sphere{1}$ and $\SU{2}\cong\sphere{3}$.

Our recursive construction for $\sphere{D}$ provides a starting point for flows on $\SU{D}$ for $D \geq 3$ and $\text{U}(D)$ for $D \geq 2$, via recursively building $\SU{D}$ from $\SU{D-1}$ and $\text{U}(1)^{2D-1}$, and $\text{U}(D)$ from $\SU{D}$ and extra angles. These decompositions provide coordinate systems on $\SU{D}$ and $\text{U}(D)$ in terms of Euler angles as shown by \citet{tilma2002generalized, tilma2004generalized}.
Working with explicit coordinate systems on $\SU{D}$ and $\text{U}(D)$ is more involved than on spheres; figuring out the necessary details such as appropriate boundary conditions for each coordinate/angle is a direction we plan to explore in the future in order to build flows on a wider variety of compact connected manifolds.

\ifanonymize
\else
\section*{Acknowledgements}
We thank Heiko Strathmann and Alex Botev for discussions and feedback. PES is partially supported by the National Science Foundation under CAREER Award 1841699 and PES and GK are partially supported by the U.S.\ Department of Energy, Office of Science, Office of Nuclear Physics under grant Contract Number DE-SC0011090.
KC is supported by the National Science Foundation under the awards ACI-1450310, OAC-1836650, and OAC-1841471 and by the Moore-Sloan data science environment at NYU\@.
\fi

\bibliographystyle{icml2020}
\setlength{\bibsep}{0.239cm}
\bibliography{references}

\begin{thebibliography}{45}
\providecommand{\natexlab}[1]{#1}
\providecommand{\url}[1]{\texttt{#1}}
\expandafter\ifx\csname urlstyle\endcsname\relax
  \providecommand{\doi}[1]{doi: #1}\else
  \providecommand{\doi}{doi: \begingroup \urlstyle{rm}\Url}\fi

\bibitem[Arnaud et~al.(2001)Arnaud, de~Freitas, and
  Gordon]{arnaud2001sequential}
Arnaud, D., de~Freitas, N., and Gordon, N.
\newblock \emph{Sequential {M}onte {C}arlo methods in practice}.
\newblock Information Science and Statistics. Springer New York, 2001.

\bibitem[Boomsma et~al.(2008)Boomsma, Mardia, Taylor, Ferkinghoff-Borg, Krogh,
  and Hamelryck]{boomsma2008generative}
Boomsma, W., Mardia, K.~V., Taylor, C.~C., Ferkinghoff-Borg, J., Krogh, A., and
  Hamelryck, T.
\newblock A generative, probabilistic model of local protein structure.
\newblock \emph{Proceedings of the National Academy of Sciences}, 105\penalty0
  (26):\penalty0 8932--8937, 2008.

\bibitem[Davidson et~al.(2018)Davidson, Falorsi, De~Cao, Kipf, and
  Tomczak]{davidson2018svae}
Davidson, T.~R., Falorsi, L., De~Cao, N., Kipf, T., and Tomczak, J.~M.
\newblock Hyperspherical variational auto-encoders.
\newblock \emph{Conference on Uncertainty in Artificial Intelligence}, 2018.

\bibitem[Dinh et~al.(2017)Dinh, Sohl-Dickstein, and Bengio]{dinh2017density}
Dinh, L., Sohl-Dickstein, J., and Bengio, S.
\newblock Density estimation using {R}eal {NVP}.
\newblock \emph{International Conference on Learning Representations}, 2017.

\bibitem[Downs(1972)]{downs1972orientation}
Downs, T.~D.
\newblock Orientation statistics.
\newblock \emph{Biometrika}, 59\penalty0 (3):\penalty0 665--676, 1972.

\bibitem[Downs \& Mardia(2002)Downs and Mardia]{downs2002circular}
Downs, T.~D. and Mardia, K.~V.
\newblock Circular regression.
\newblock \emph{Biometrika}, 89\penalty0 (3):\penalty0 683--698, 2002.

\bibitem[Durkan et~al.(2019{\natexlab{a}})Durkan, Bekasov, Murray, and
  Papamakarios]{durkan2019cubic}
Durkan, C., Bekasov, A., Murray, I., and Papamakarios, G.
\newblock Cubic-spline flows.
\newblock \emph{ICML workshop on Invertible Neural Networks and Normalizing
  Flows}, 2019{\natexlab{a}}.

\bibitem[Durkan et~al.(2019{\natexlab{b}})Durkan, Bekasov, Murray, and
  Papamakarios]{durkan2019neural}
Durkan, C., Bekasov, A., Murray, I., and Papamakarios, G.
\newblock Neural spline flows.
\newblock \emph{Advances in Neural Information Processing Systems},
  2019{\natexlab{b}}.

\bibitem[Falorsi et~al.(2018)Falorsi, de~Haan, Davidson, De~Cao, Weiler,
  Forr{\'e}, and Cohen]{falorsi2018explorations}
Falorsi, L., de~Haan, P., Davidson, T.~R., De~Cao, N., Weiler, M., Forr{\'e},
  P., and Cohen, T.~S.
\newblock Explorations in homeomorphic variational auto-encoding.
\newblock \emph{ICML workshop on Theoretical Foundations and Applications of
  Deep Generative Models}, 2018.

\bibitem[Falorsi et~al.(2019)Falorsi, de~Haan, Davidson, and
  Forr{\'e}]{falorsi2019reparameterizing}
Falorsi, L., de~Haan, P., Davidson, T.~R., and Forr{\'e}, P.
\newblock Reparameterizing distributions on {L}ie groups.
\newblock \emph{International Conference on Artificial Intelligence and
  Statistics}, 2019.

\bibitem[Feiten et~al.(2013)Feiten, Lang, and Hirche]{feiten2013rigid}
Feiten, W., Lang, M., and Hirche, S.
\newblock Rigid motion estimation using mixtures of projected {G}aussians.
\newblock \emph{International Conference on Information Fusion}, 2013.

\bibitem[Fisher(1953)]{fisher1953dispersion}
Fisher, R.
\newblock Dispersion on a sphere.
\newblock \emph{Proceedings of the Royal Society of London. Series A,
  Mathematical and Physical Sciences}, 217\penalty0 (1130):\penalty0 295--305,
  1953.

\bibitem[Gatto \& Jammalamadaka(2007)Gatto and
  Jammalamadaka]{gatto2007generalized}
Gatto, R. and Jammalamadaka, S.~R.
\newblock The generalized von {M}ises distribution.
\newblock \emph{Statistical Methodology}, 4\penalty0 (3):\penalty0 341--353,
  2007.

\bibitem[Gemici et~al.(2016)Gemici, Rezende, and
  Mohamed]{gemici2016normalizing}
Gemici, M.~C., Rezende, D.~J., and Mohamed, S.
\newblock Normalizing flows on {R}iemannian manifolds.
\newblock \emph{arXiv preprint arXiv:1611.02304}, 2016.

\bibitem[Hamelryck et~al.(2006)Hamelryck, Kent, and
  Krogh]{hamelryck2006sampling}
Hamelryck, T., Kent, J.~T., and Krogh, A.
\newblock Sampling realistic protein conformations using local structural bias.
\newblock \emph{PLoS Computational Biology}, 2\penalty0 (9), 2006.

\bibitem[Jakob(2012)]{jakob2012numerically}
Jakob, W.
\newblock Numerically stable sampling of the von {M}ises {F}isher distribution
  on {$S^2$} (and other tricks).
\newblock Technical report, Interactive Geometry Lab, ETH Z{\"u}rich, 2012.

\bibitem[Jona-Lasinio et~al.(2012)Jona-Lasinio, Gelfand, and
  Jona-Lasinio]{jona2012spatial}
Jona-Lasinio, G., Gelfand, A., and Jona-Lasinio, M.
\newblock Spatial analysis of wave direction data using wrapped {G}aussian
  processes.
\newblock \emph{The Annals of Applied Statistics}, 6\penalty0 (4):\penalty0
  1478--1498, 2012.

\bibitem[Kato \& McCullagh(2015)Kato and McCullagh]{kato2015m}
Kato, S. and McCullagh, P.
\newblock M\"obius transformation and a {C}auchy family on the sphere.
\newblock \emph{arXiv preprint arXiv:1510.07679}, 2015.

\bibitem[Kato et~al.(2008)Kato, Shimizu, and Shieh]{kato2008circular}
Kato, S., Shimizu, K., and Shieh, G.~S.
\newblock A circular--circular regression model.
\newblock \emph{Statistica Sinica}, 18:\penalty0 633--645, 2008.

\bibitem[Kent(1982)]{kent1982fisherbingham}
Kent, J.~T.
\newblock The {F}isher--{B}ingham distribution on the sphere.
\newblock \emph{Journal of the Royal Statistical Society. Series B
  (Methodological)}, 44\penalty0 (1):\penalty0 71--80, 1982.

\bibitem[Kingma \& Ba(2015)Kingma and Ba]{kingma2015adam}
Kingma, D.~P. and Ba, J.
\newblock Adam: A method for stochastic optimization.
\newblock \emph{International Conference for Learning Representations}, 2015.

\bibitem[Kingma et~al.(2016)Kingma, Salimans, Jozefowicz, Chen, Sutskever, and
  Welling]{kingma2016iaf}
Kingma, D.~P., Salimans, T., Jozefowicz, R., Chen, X., Sutskever, I., and
  Welling, M.
\newblock Improved variational inference with inverse autoregressive flow.
\newblock \emph{Advances in Neural Information Processing Systems}, 2016.

\bibitem[Kobayashi \& Nomizu(1963)Kobayashi and
  Nomizu]{kobayashi1963foundations}
Kobayashi, S. and Nomizu, K.
\newblock \emph{Foundations of differential geometry}, volume~1.
\newblock Interscience Publishers, 1963.

\bibitem[Liu(2008)]{liu2008monte}
Liu, J.~S.
\newblock \emph{{M}onte {C}arlo strategies in scientific computing}.
\newblock Springer Science \& Business Media, 2008.

\bibitem[Mardia \& Jupp(2009)Mardia and Jupp]{mardia2009directional}
Mardia, K.~V. and Jupp, P.~E.
\newblock \emph{Directional statistics}.
\newblock Wiley Series in Probability and Statistics. John Wiley \& Sons, 2009.

\bibitem[Mardia et~al.(2007)Mardia, Taylor, and Subramaniam]{mardia2007protein}
Mardia, K.~V., Taylor, C.~C., and Subramaniam, G.~K.
\newblock Protein bioinformatics and mixtures of bivariate von {M}ises
  distributions for angular data.
\newblock \emph{Biometrics}, 63\penalty0 (2):\penalty0 505--512, 2007.

\bibitem[Mardia et~al.(2008)Mardia, Hughes, Taylor, and
  Singh]{mardia2008multivariate}
Mardia, K.~V., Hughes, G., Taylor, C.~C., and Singh, H.
\newblock A multivariate von {M}ises distribution with applications to
  bioinformatics.
\newblock \emph{The Canadian Journal of Statistics}, 36\penalty0 (1):\penalty0
  99--109, 2008.

\bibitem[Mathieu et~al.(2019)Mathieu, Le~Lan, Maddison, Tomioka, and
  Teh]{mathieu2019poincare}
Mathieu, E., Le~Lan, C., Maddison, C.~J., Tomioka, R., and Teh, Y.~W.
\newblock Continuous hierarchical representations with {P}oincar\'{e}
  variational auto-encoders.
\newblock \emph{Advances in Neural Information Processing Systems}, 2019.

\bibitem[M{\"u}ller et~al.(2019)M{\"u}ller, McWilliams, Rousselle, Gross, and
  Nov{\'a}k]{muller2019neural}
M{\"u}ller, T., McWilliams, B., Rousselle, F., Gross, M., and Nov{\'a}k, J.
\newblock Neural importance sampling.
\newblock \emph{ACM Transactions on Graphics}, 38\penalty0 (5):\penalty0 1--19,
  2019.

\bibitem[Navarro et~al.(2017)Navarro, Frellsen, and
  Turner]{navarro2017multivariate}
Navarro, A. K.~W., Frellsen, J., and Turner, R.~E.
\newblock The multivariate generalised von {M}ises distribution: Inference and
  applications.
\newblock \emph{AAAI Conference on Artificial Intelligence}, 2017.

\bibitem[Papamakarios et~al.(2017)Papamakarios, Pavlakou, and
  Murray]{papamakarios2017masked}
Papamakarios, G., Pavlakou, T., and Murray, I.
\newblock Masked autoregressive flow for density estimation.
\newblock \emph{Advances in Neural Information Processing Systems}, 2017.

\bibitem[Papamakarios et~al.(2019)Papamakarios, Nalisnick, Rezende, Mohamed,
  and Lakshminarayanan]{papamakarios2019normalizing}
Papamakarios, G., Nalisnick, E., Rezende, D.~J., Mohamed, S., and
  Lakshminarayanan, B.
\newblock Normalizing flows for probabilistic modeling and inference.
\newblock \emph{arXiv preprint arXiv:1912.02762}, 2019.

\bibitem[Peel et~al.(2001)Peel, Whiten, and McLachlan]{peel2001mixtures}
Peel, D., Whiten, W.~J., and McLachlan, G.~J.
\newblock Fitting mixtures of {K}ent distributions to aid in joint set
  identification.
\newblock \emph{Journal of the American Statistical Association}, 96\penalty0
  (453):\penalty0 56--63, 2001.

\bibitem[Rezende \& Mohamed(2015)Rezende and Mohamed]{rezende2015variational}
Rezende, D.~J. and Mohamed, S.
\newblock Variational inference with normalizing flows.
\newblock \emph{International Conference on Machine Learning}, 2015.

\bibitem[Sei(2013)]{sei2013jacobian}
Sei, T.
\newblock A {J}acobian inequality for gradient maps on the sphere and its
  application to directional statistics.
\newblock \emph{Communications in Statistics---Theory and Methods}, 42\penalty0
  (14):\penalty0 2525--2542, 2013.

\bibitem[Senanayake \& Ramos(2018)Senanayake and
  Ramos]{senanayake2018directional}
Senanayake, R. and Ramos, F.
\newblock Directional grid maps: modeling multimodal angular uncertainty in
  dynamic environments.
\newblock \emph{IEEE/RSJ International Conference on Intelligent Robots and
  Systems}, 2018.

\bibitem[Shapovalov \& Dunbrack~Jr(2011)Shapovalov and
  Dunbrack~Jr]{shapovalov2011smoothed}
Shapovalov, M.~V. and Dunbrack~Jr, R.~L.
\newblock A smoothed backbone-dependent rotamer library for proteins derived
  from adaptive kernel density estimates and regressions.
\newblock \emph{Structure}, 19\penalty0 (6):\penalty0 844--858, 2011.

\bibitem[Tabak \& Turner(2013)Tabak and Turner]{tabak2013family}
Tabak, E.~G. and Turner, C.~V.
\newblock A family of nonparametric density estimation algorithms.
\newblock \emph{Communications on Pure and Applied Mathematics}, 66\penalty0
  (2):\penalty0 145--164, 2013.

\bibitem[Tilma \& Sudarshan(2002)Tilma and Sudarshan]{tilma2002generalized}
Tilma, T. and Sudarshan, E.
\newblock Generalized {E}uler angle parametrization for {$\mathit{SU}(N)$}.
\newblock \emph{Journal of Physics A: Mathematical and General}, 35\penalty0
  (48):\penalty0 10467, 2002.

\bibitem[Tilma \& Sudarshan(2004)Tilma and Sudarshan]{tilma2004generalized}
Tilma, T. and Sudarshan, E.
\newblock Generalized {E}uler angle parameterization for {$\mathit{U}(N)$} with
  applications to {$\mathit{SU}(N)$} coset volume measures.
\newblock \emph{Journal of Geometry and Physics}, 52\penalty0 (3):\penalty0
  263--283, 2004.

\bibitem[Wang \& Gelfand(2013)Wang and Gelfand]{wang2013projectednormal}
Wang, F. and Gelfand, A.~E.
\newblock Directional data analysis under the general projected normal
  distribution.
\newblock \emph{Statistical Methodology}, 10\penalty0 (1):\penalty0 113--127,
  2013.

\bibitem[Wang(2013)]{wang2013extensions}
Wang, M.
\newblock \emph{Extensions of probability distributions on torus, cylinder and
  disc}.
\newblock {P}h{D} thesis, Graduate School of Science and Technology, Keio
  University, 2013.

\bibitem[Wang \& Wang(2019)Wang and Wang]{wang2019riemannian}
Wang, P.~Z. and Wang, W.~Y.
\newblock {R}iemannian normalizing flow on variational {W}asserstein
  autoencoder for text modeling.
\newblock \emph{Conference of the North American Chapter of the Association for
  Computational Linguistics: Human Language Technologies}, 2019.

\bibitem[Wang et~al.(2019)Wang, Cheng, Li, Zhu, and Zhang]{wang2019wasserstein}
Wang, Z., Cheng, S., Li, Y., Zhu, J., and Zhang, B.
\newblock A {W}asserstein minimum velocity approach to learning unnormalized
  models.
\newblock \emph{Symposium on Advances in Approximate Bayesian Inference}, 2019.

\bibitem[Woit(2017)]{woit2017quantum}
Woit, P.
\newblock \emph{Quantum theory, groups and representations: An introduction}.
\newblock Springer, 2017.

\end{thebibliography}

\clearpage
\appendix
\section{Density Transformations on Manifolds} \label{sec:volume.elements}

In this section, we explain how to update the density of a distribution transformed from one Riemannian manifold to another by a smooth map. We only consider the case where both manifolds are sub-manifolds of Euclidean spaces.

Let $\MM$ and $\NN$ be $D$-dimensional manifolds embedded into Euclidean spaces $\R^m$ and $\R^n$ respectively. For example, $\MM$ and $\NN$ could be $\sphere{D}$ embedded in $\R^{D+1}$ as in \Cref{sec:ar:sphere}. Both manifolds inherit a Riemannian metric from their embedding spaces.
Let $T$ be a smooth injective map $T:\MM\rightarrow\NN$.
We will assume that $T$ can be extended to a smooth map between open neighbourhoods of the embedding spaces that contain $\MM$ and $\NN$, and that we have chosen such an extension.
For example, the exponential-map flow in \Cref{eq:exp_flow} can be written using the coordinates of the embedding space $\R^{D+1}$, and can thus be extended to open neighbourhoods of the embedding spaces as desired.

In what follows, we will use the fact that if $u_1,\ldots,u_D$ are vectors in $\R^n$, then the volume of the parallelepiped with sides $u_1,\ldots,u_D$ is $\sqrt{\det\br{ U^\top U}}$, where $U$ is the $n\times D$ matrix with column vectors $u_1,\ldots,u_D$. If $u_1,\ldots,u_D$ form an orthonormal system, this volume is $1$.

Let $\pi:\MM\rightarrow\R^+$ be a density on $\MM$. This defines a distribution on $\MM$ and we can use $T$ to transform it into a distribution on $\NN$. Let $p:\NN\rightarrow\R^+$ be the density of the transformed distribution. We are interested in computing $p$ assuming we know $\pi$. Let $x$ be a point on $\MM$, and $e_1,\ldots,e_D$ be an orthonormal basis of the tangent space $\text{T}_x\MM$\@. Define $E$ to be the $m\times D$ matrix with $i$-th column vector $e_i$.
Let $J$ be the $n\times m$ Jacobian of $T$, where $T$ is seen as a map between open sets in $\R^m$ and $\R^n$.
The tangent map of $T$ at $x$ transforms each $e_i$ to $Je_i$, and the matrix that collects all transformed vectors in its columns is $JE$.
Hence, the volume of the parallelepiped with sides the transformed vectors is $\sqrt{\det\br{ \br{JE}^\top J E}} = \sqrt{\det\br{ E^\top J^\top J E}}$. Therefore, the density $p$ is given by
\begin{equation}\label{eq:riemannian_det_jac}
    p(T(x)) = \frac{\pi(x)}{\sqrt{\det\br{ E^\top J^\top J E}}}.
\end{equation}
In the special case where $\mathcal{M}=\mathcal{N}=\R^D$ and $m=n=D$, the matrix $E$ is an orthogonal matrix, and the above reduces to the familiar density update in \Cref{eq:flow_jac}.

\subsection{The case of $\Tcs:\sphere{D-1}\times [-1, 1]\rightarrow \sphere{D}$}\label{sec:tcs}
In this section, we specialize to $\MM=\sphere{D-1}\times (-1, 1)$ and $\NN=\sphere{D}$ with $D\ge 2$. In particular, we will prove:
\begin{prop}\label{prop:tcs_pull-back}
Let $\pi$ be a density $\pi:\sphere{D-1}\times (-1, 1)\rightarrow\R^+$. Let $p:\sphere{D}\rightarrow\R^+$ be the density of the transformed distribution under $\Tcs$. Then:
\begin{equation}
    p(\Tcs(z, r)) = \frac{\pi(z,r)}{(1-r^2)^{\frac{D}{2}-1}}.
\end{equation}
\end{prop}
\begin{proof}
The sphere $\sphere{D-1}$ is embedded in $\R^D$. This gives us an embedding of $\MM$ in $\R^{D+1}$. The map $\Tcs$, introduced in \Cref{eq:flow_sd_3}, is easily extended to a map $\R^D\times (-1, 1)\rightarrow\R^{D+1}$ using the same formula $\Tcs(z, r) = (\sqrt{1-r^2}z, r)$.
Its Jacobian is an upper triangular $D+1$ by $D+1$ matrix
\begin{equation*}
J=
\begin{bmatrix}
\sqrt{1-r^2} & 0 & \cdots & \frac{-x_1r}{\sqrt{1-r^2}} \\
0 & \sqrt{1-r^2} &  & \frac{-x_2r}{\sqrt{1-r^2}} \\
\vdots & & \ddots & \vdots \\
0 & \cdots & & 1
\end{bmatrix}.
\end{equation*}

We will use a symmetry argument to simplify the computation of the determinant in \Cref{eq:riemannian_det_jac}. Let $G$ be a rotation of $\R^{D+1}$ that leaves the last coordinate invariant.\footnote{The set of such rotations is the group of rotations of $\R^D$ embedded in $\R^{D+1}.$} Note that for any point $x\in\R^{D+1}$, we have $\Tcs(Gx) = G\Tcs(x)$.
This means that the Jacobian $J$ transforms as a function of $x$ as $J(Gx) = GJ(x)G^\top$.
Note also that if $x$ is in $\MM$, and $E(x)$ is a matrix where the column vectors form a basis of the tangent space at $x$, then $Gx$ is also in $\MM$, and $GE(x)$ is a matrix where the column vectors form a basis of the tangent space at $Gx$. So we can choose $E(Gx)=GE(x)$. With that choice
\begin{align*}
    & \det\br{E(Gx)^\top J(Gx)^\top J(Gx)E(Gx)} \\
    &= \det\br{ E(x)^\top G^\top GJ(x)^\top G^\top GJ(x)G^\top GE(x)} \\
    &= \det\br{E(x)J(x)^\top J(x)E(x)}.
\end{align*}
Since for any $x\in\MM$, we can always choose $G$ such that $Gx$ is of the form $(\sqrt{1-r^2}, 0,\ldots,0,r)$, we can restrict ourselves to this case. For such a choice, the Jacobian simplifies to
\begin{equation*}
J=
\begin{bmatrix}
\sqrt{1-r^2} & 0 & \cdots & \frac{-r}{\sqrt{1-r^2}} \\
0 & \sqrt{1-r^2} &  & 0 \\
\vdots& & \ddots & \vdots \\
0 & \cdots & & 1
\end{bmatrix}.
\end{equation*}
For $E$, we can simply choose the $D+1$ by $D$ matrix made by removing the first column from the identity matrix. Then $JE$ is equal to $J$ with the first column removed:
\begin{equation*}
    JE = \begin{bmatrix}
0 & 0 & \cdots & \frac{-r}{\sqrt{1-r^2}} \\
\sqrt{1-r^2} & 0 &  & 0 \\
0 & \sqrt{1-r^2} &  & 0 \\
\vdots &  & \ddots & \vdots \\
0 & \cdots & & 1
\end{bmatrix}.
\end{equation*}
The product $(JE)^\top JE$ is simply a diagonal matrix of size $D$ with diagonal $(1-r^2,\ldots,1-r^2,\frac{1}{1-r^2})$. Taking the determinant and applying \Cref{eq:riemannian_det_jac} concludes the proof.
\end{proof}

At first, \Cref{prop:tcs_pull-back} might seem worrying since the density ratio in that proposition vanishes when $r$ is $-1$ or $1$. So, as $r$ approaches the boundary of the interval $[-1, 1]$, it seems that the correction term to the density will tend to infinity and lead to numerical instability.

What saves us is that we do not use $\Tcs$ on its own, and instead combine it with a particular flow transformation on $\sphere{D-1}\times[-1,1]$ and the inverse $\Tsc$, as shown in \Cref{eq:flow_sd_1,eq:flow_sd_2,eq:flow_sd_3}. In these formulas, the map $g$ is a spline on the interval $[-1, 1]$ which maps $-1$ to $-1$, $1$ to $1$, and has strictly positive slopes $g'(-1)$ and $g'(1)$. Looking only at $-1$ (the case $1$ can be similarly dealt with), this means
\begin{equation*}
g(-1+\epsilon)\approx -1 + g'(-1)\epsilon.   
\end{equation*}
As $\epsilon$ goes to $0$, the density corrections coming from $\Tcs$ and $\Tsc$ combine to 
\begin{equation*}
\br{\frac{D}{2}-1}\log\frac{1-g(-1+\epsilon)^2}{1-(-1+\epsilon)^2},    
\end{equation*}
which is equivalent to
\begin{equation*}
\br{\frac{D}{2}-1}\log g'(-1)
\end{equation*}
as $\epsilon$ goes to $0$. In particular, the terms that tend to infinity cancel each other, and the flow is well-behaved. When implementing the flow, numerical stability is achieved by not adding the terms that cancel each other. Finally, we note a subtle point about what we proved: the sequence of transformations $\Tcs\circ T_{c\rightarrow c}\circ\Tsc$ will transform a distribution with finite density into another distribution with finite density, but we do not guarantee that the resulting density will be continuous.

\section{Detailed Diagram of Recursive Flow on $\sphere{D}$}

In \Cref{fig:flow:s2:diagram:detailed}, we provide an illustration of the recursive construction in \Cref{eq:flow_sd_1,eq:flow_sd_2,eq:flow_sd_3}, showing the specific wiring order of the conditional maps inside the flow. This order is the one implied by the recursion. In general, any other order can be used, or a composition of autoregressive flows with multiple orders.

\begin{figure}[htb]
\begin{center}
    \begin{subfigure}[t]{\textwidth}
        \includegraphics[width=0.45\textwidth]{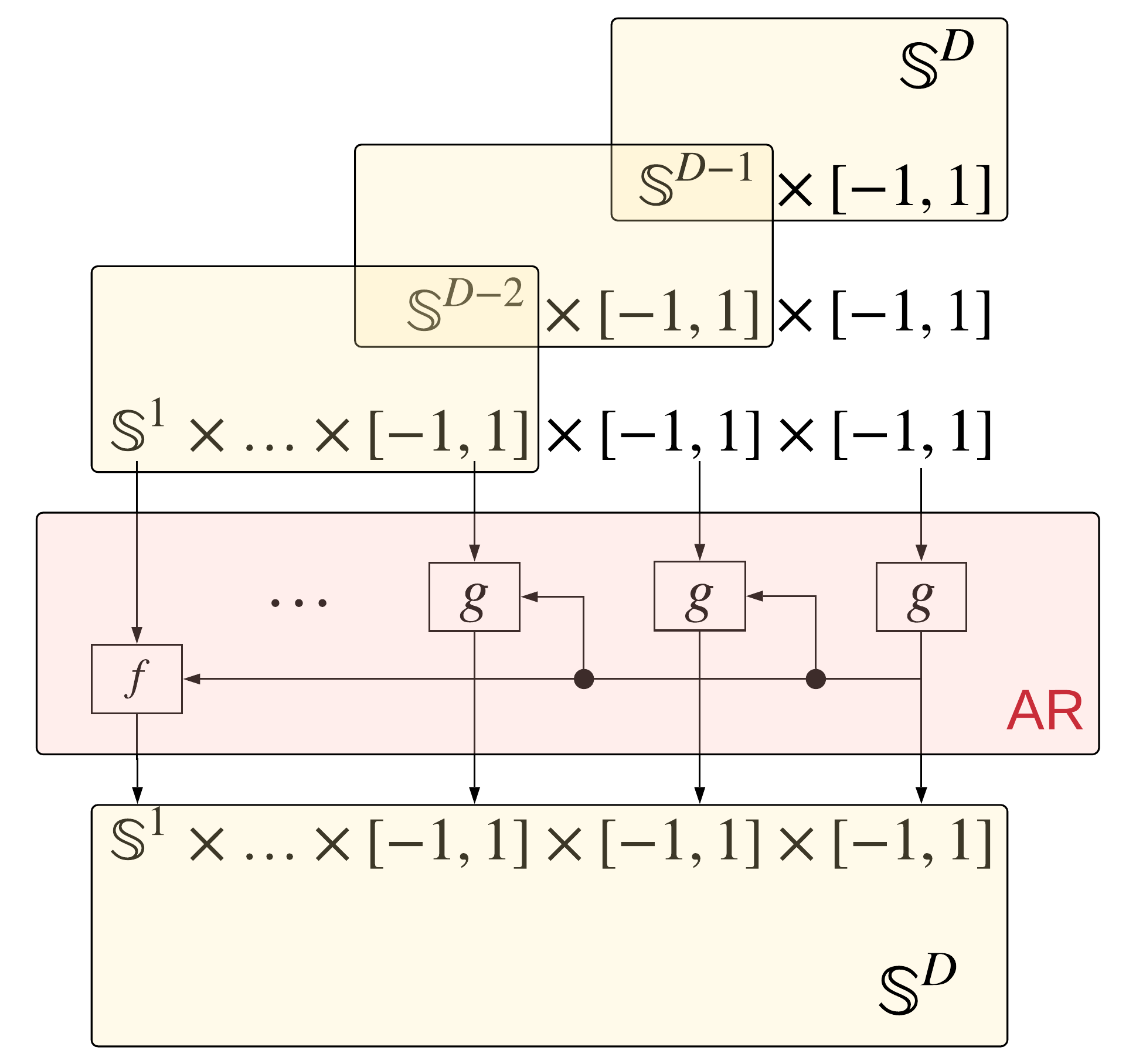}
    \end{subfigure}
\end{center}
\caption{Detailed illustration of the recursive flow on the sphere $\sphere{D}$ showing the explicit wiring of the conditional flows. The sphere $\sphere{D}$ is recursively transformed to the cylinder $\sphere{1}\times[-1, 1]^{D-1}$, then an autoregressive flow is applied to the cylinder, and finally the cylinder is transformed back to the sphere.}\label{fig:flow:s2:diagram:detailed}
\end{figure}

\begin{figure*}[p]
\begin{center}
    \begin{subfigure}[t]{\textwidth}
        \centering
        \includegraphics[width=\textwidth]{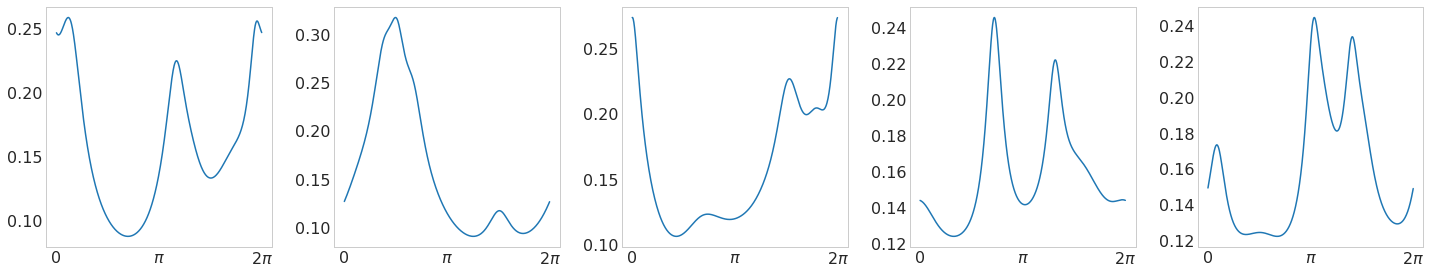}
    \end{subfigure}
\end{center}
\caption{Probability density functions of convex combinations of $15$ M\"obius transformations applied to a uniform base distribution on the circle $\sphere{1}$. Each of these distributions required $30=15\times 2$ parameters.}\label{fig:expressive_moebius}
\end{figure*}

\begin{table*}[p]
\begin{center}
\begin{tabular}{@{}|c|c|l|@{}}
\toprule
\textbf{Target} & \textbf{Expression} & \multicolumn{1}{c|}{\textbf{Parameters}} \\ \midrule
Unimodal & $p_A(\theta_1, \theta_2) \propto \exp [  \cos(\theta_1 - \phi_1) + \cos(\theta_2 - \phi_2)  ]$ & $\phi=(4.18, 5.96)$ \\ \midrule
Multi-modal & $p_B(\theta_1, \theta_2) \propto \frac{1}{3} \sum_{i=1}^3 p_A(\theta_1, \theta_2; \phi_i)$ & $\phi=\{(0.21, 2.85), (1.89, 6.18), (3.77, 1.56)\}$ \\ \midrule
Correlated & $p_C(\theta_1, \theta_2) \propto \exp [ \cos(\theta_1 + \theta_2 - \phi) ]$ & $\phi=1.94$ \\ \bottomrule
\end{tabular}
\caption{Target densities used for experiments on $\torus{2}$.\label{table.tori.target}}
\end{center}
\end{table*}

\begin{figure*}[p]
\begin{center}
\centerline{\includegraphics[width=1.02\linewidth]{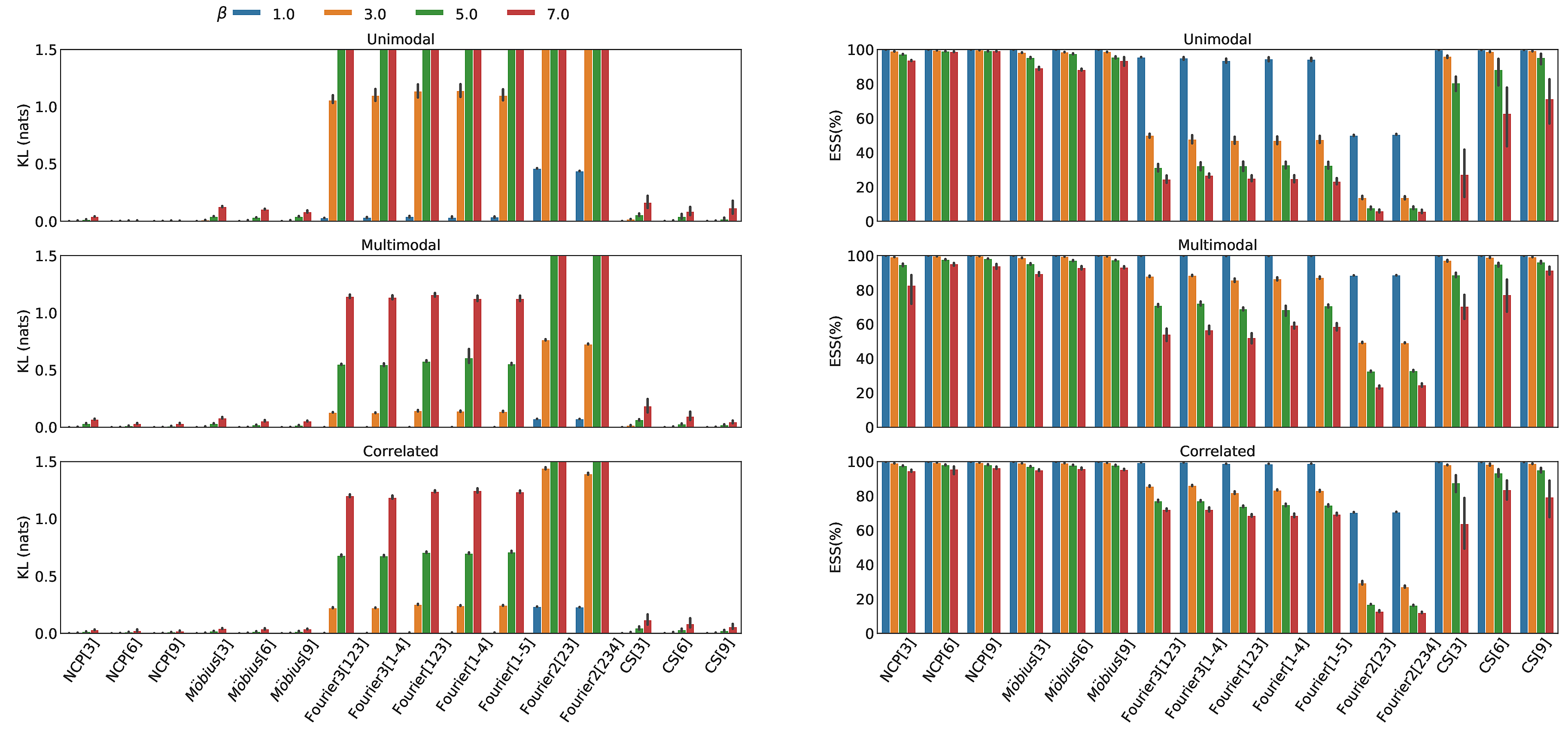}}
\caption{Same as \Cref{fig.torus}, with KL and values for the Fourier transforms added. For the Fourier models, the numbers between brackets represent used frequencies, and a number before the bracket means each frequency was repeated. For example, Fourier$3[1-4]$ is a Fourier model with $12$ frequencies: $3$ frequencies of $k$ for each $k=1,\ldots,4$. 
}
\label{fig:results_torus_all}
\end{center}
\end{figure*}

\section{Examples of M\"obius Transformations}

To illustrate the kind of densities we get on $\mathbb{S}^1$ using the M\"obius flows, we show a few random examples in \Cref{fig:expressive_moebius}.

\section{Fourier Transformations on $\sphere{1}$}

Another family of circle transformations that we considered are \textit{Fourier transformations}, defined by
\begin{align}
    f_{\alpha, \phi, w}(\theta) &= \theta + \sum_i \frac{\alpha_i}{w_i} \sin( w_i \theta - \phi_i ) + \mu,
\end{align}
where $\mu = \sum_i \alpha_i \sin( \phi_i )$, $w_i \in \mathbb{Z}$, $\phi_i \in [0, 2\pi]$ and $\sum_i |\alpha_i| \leq 1$. The integers $w_i$ are fixed frequencies in the Fourier basis.

We found empirically that this family of transformations is not competitive with the other transformations considered in this paper, especially for highly concentrated densities as shown in \Cref{fig:results_torus_all}.

\section{Polynomial Exponential Map}\label{sec:polexpmap}

The polynomial exponential map of \citet{sei2013jacobian} is the exponential-map flow built using the scalar field
\begin{align}
    \phi(x) &= \mu^\top x + x^\top Ax,
\end{align}
where $x \in \sphere{D}$ in the coordinates of the embedding space $\R^{D+1}$. The parameters $\mu$ and $A$ must satisfy the constraint $\norm{\mu}_1 + \norm{A}_1 \leq 1$ where $\norm{\cdot}_1$ is the elementwise $\ell_1$ norm.

\section{Target Densities Used in Experiments}

For the experiments on the torus $\torus{2}$, we used targets built from densities in the von Mises family as shown in \Cref{table.tori.target}.

On the sphere $\sphere{2}$, the target was a mixture of the form
\begin{align}
    p(x) &\propto \sum_{k=1}^4 e^{10 x^\top T_{s \rightarrow e}(\mu_k)},
\end{align}
where $\mu_1 = (0.7, 1.5)$, $\mu_2=(-1, 1)$, $\mu_3=(0.6, 0.5)$, $\mu_4 = (-0.7, 4)$,
$T_{s \rightarrow e}$ maps from spherical to Euclidean coordinates, and $x\in\mathbb{R}^3$ is a point on the embedded sphere in Euclidean coordinates.

On $\SU{2} \cong \sphere{3}$, the target was a mixture of the same form where $\mu_1 = (1.7, -1.5, 2.3)$, $\mu_2=(-3.0, 1.0, 3.0)$, $\mu_3=(0.6, -2.6, 4.5)$, $\mu_4 = (-2.5, 3.0, 5.0)$, 
and $x\in\mathbb{R}^4$ is a point on the embedded sphere in Euclidean coordinates.

\section{Misaligned Density on $\sphere{2}$}

The recursive formulas shown in \Cref{eq:flow_sd_1,eq:flow_sd_2,eq:flow_sd_3} require choosing a sequence of axes in order to construct the cylindrical coordinate system. This may introduce artifacts to the density related to this choice of axes.
To test if this results in numerical problems, we compare the flow from \Cref{eq:flow_sd_1,eq:flow_sd_2,eq:flow_sd_3} on a target density that forms a non-axis-aligned ring against a composition of the same flow with a learned rotation.

The results of this experiment are shown in \Cref{fig:slanted}. We compared both large ($K_s=32$, $K_m=12$) and small ($K_s=3$, $K_m=3$) versions of the auto-regressive M\"obius-Spline flow and observed no significant differences between the two models on $\sphere{2}$. 

More experiments would be necessary to investigate this potential effect in higher dimensions.

\section{NCP as a complex M\"obius transformation}\label{sec:ncp_mobius_equiv}

For a general M\"obius transformation 
\begin{align}
    f (z) &= \frac{a z + b}{c z + d},
\end{align}
where $a, b, c, d, z \in \mathbb{C}$, to define a diffeomorphism on $\sphere{1}$ it must be constrained to be of the form
\begin{align} 
h (z) &= \frac{z - a}{1 - \bar{a} z}.
\end{align}
This form ensures that $h(z) \bar{h}(z) = 1$ if $z \bar{z} = 1$ and has
two real-valued free parameters $\Re(a)$ and $\Im(a)$.

In what follows we show that for the choice $\Im(a) = 0$ and $\Re(a) = - \frac{1 - \alpha}{1 + \alpha}$, the transformation $h (z)$ is equivalent to an
NCP transform $w = 2 \arctan \left( \alpha \tan \left( \frac{\theta}{2}
\right) + \beta \right)$ with scale parameter $\alpha$ and offset parameter $\beta
= 0$ (assuming $w, \theta \in (- \pi, \pi)$). If we define $\theta$ via $z =
e^{i \theta}$, the goal is to show that $w$ defined via
\begin{align}
  e^{i w} &= \frac{e^{i \theta} + \frac{1 - \alpha}{1 + \alpha}}{1 + \frac{1 - \alpha}{1 + \alpha}
  e^{i \theta}}, \label{eq.moebius.ncp}
\end{align}
follows the NCP transformation rule \begin{equation*}
w = 2 \arctan \left( \alpha \tan
\left( \frac{\theta}{2} \right) \right) \mod 2 \pi.
\end{equation*}
We begin by
expanding \Cref{eq.moebius.ncp} in terms of more basic trigonometric
quantities,
\begin{align*}
  e^{i w} &= \frac{e^{i \theta} + \frac{1 - \alpha}{1 + \alpha}}{1 + \frac{1 - \alpha}{1 +
  \alpha} e^{i \theta}} & \\
  &= \frac{e^{i \theta} + \frac{1 - \alpha}{1 + \alpha}}{1 + \frac{1 - \alpha}{1 + \alpha} e^{i
  \theta}} \frac{1 + \frac{1 - \alpha}{1 + \alpha} e^{- i \theta}}{1 + \frac{1 - \alpha}{1 +
  \alpha} e^{- i \theta}} & \\
  &= \frac{e^{i \theta} + 2 \frac{1 - \alpha}{1 + \alpha} + \left( \frac{1 - \alpha}{1 + \alpha}
  \right)^2 e^{- i \theta}}{\frac{ 2 (\cos (\theta) + \alpha^2 (- \cos
  (\theta)) + \alpha^2 + 1)}{(\alpha + 1)^2}}\\
  &= \frac{\left( e^{i \frac{\theta}{2}} +
  \frac{1 - \alpha}{1 + \alpha} e^{- i \frac{\theta}{2}} \right)^2}{\frac{ 2
  (\cos (\theta) + \alpha^2 (- \cos (\theta)) + \alpha^2 + 1)}{(\alpha + 1)^2}}\\
  &= \frac{1}{2} \frac{\left( (1 + \alpha) e^{i \frac{\theta}{2}} + (1 - \alpha) e^{-
  i \frac{\theta}{2}} \right)^2}{ \cos (\theta) + \alpha^2 (- \cos
  (\theta)) + \alpha^2 + 1}\\
  &= \frac{2 \left( \cos \left( \frac{\theta}{2} \right) +
  i\alpha \sin \left( \frac{\theta}{2} \right) \right)^2}{\cos (\theta)
  + \alpha^2 (- \cos (\theta)) + \alpha^2 + 1}.
\end{align*}
In order to isolate $w$, only the numerator $v = 2 \left( \cos \left(
\frac{\theta}{2} \right) + i\alpha \sin \left( \frac{\theta}{2} \right) \right)^2$
of the expression above matters as we are only interested in ratios of the
imaginary and real parts of this expression, $\tan (w) = \frac{\Im(v)}{\Re(v)}$. The numerator can be expanded as
\begin{align*} 
\frac{2 \left( \cos \left( \frac{\theta}{2} \right) + i\alpha \sin \left(
   \frac{\theta}{2} \right) \right)^2}{2 \cos^2 \left( \frac{\theta}{2}
   \right)} &= - \alpha^2 \tan^2 \left( \frac{\theta}{2} \right)\\ 
   &+ 2 \alpha \tan \left(
   \frac{\theta}{2} \right) {\bf i} + 1,
\end{align*}
from which we conclude,
\begin{align} 
\tan (w) = \frac{\Im(v)}{\Re(v)} = \frac{2 \alpha \tan \left(
   \frac{\theta}{2} \right)}{1 - \alpha^2 \tan^2 \left( \frac{\theta}{2} \right)} .
\end{align}
Using the trigonometric formula $\tan (2 x) = \frac{2 \tan
(x)}{1 - \tan (x)^2}$, we arrive at the final result
\begin{align*}
\tan \left( \frac{w}{2} \right) &= \alpha \tan \left( \frac{\theta}{2} \right)\\
  &\Leftrightarrow\\
w &= 2 \arctan \left( \alpha \tan \left( \frac{\theta}{2}
  \right) \right) \mod 2 \pi.
\end{align*}

\section{Application: Multi-Link Robot Arm}\label{sec:robot_arm}

As a concrete application of flows on tori, we consider the problem of approximating the posterior density over joint angles $\theta_{1, \ldots, 6}$ of a $6$-link $2$D robot arm, given (soft) constraints on the position of the tip of the arm. The possible configurations of this arm are points in $\torus{6}$.
The position $r_k$ of a joint $k=1,\ldots,6$ of the robot arm is given by
\begin{align}
    r_k &= r_{k-1} +  \br{l_k\cos\br{\sum_{j \leq k}\theta_j}, l_k\sin\br{\sum_{j \leq k}\theta_j}},\nonumber
\end{align}
where $r_0= (0, 0)$ is the position where the arm is affixed, $l_k=0.2$ is the length of the $k$-th link, and $\theta_k$ is the angle of the $k$-th link in a local reference frame.
The constraint on the position of the tip of the arm, $r_6$, is expressed in the form of a Gaussian-mixture likelihood $p(r_6\g \theta_{1, \ldots, 6})$ with two components. The prior $p(\theta_{1, \ldots, 6})$ is taken to be a uniform distribution on $\torus{6}$.  
The experimental results are illustrated in \Cref{fig:robot_arm}.

\section{Application: Learning from samples}\label{sec:globe}

In most of the experiments shown on this paper, we trained the models to fit a target density known up to a normalization constant (i.e.\ an inference problem). In this experiment we train our flow directly on data samples instead.

Training a flow-based model from data samples via maximum likelihood requires an explicit computation of the inverse map as shown in \Cref{eq:flow_jac}. 
To demonstrate this is feasible with data coming from a non-trivial target density on the sphere $\sphere{2}$ (i.e.\ that would require a large number of mixture components from simpler densities such as von Mises), we created a dataset of samples on the sphere coming from a density shaped as Earth's continental map as shown in \Cref{fig:globe} (left).

We trained a flow built from stacking two autoregressive flows. Each flow in the stack used circular splines and standard splines on the interval. The model was trained to maximize the likelihood of the dataset for $100{,}000$ training steps. Both splines used $K_s=80$ segments. The neural networks producing the spline parameters are the same as for the other experiments.
In \Cref{fig:globe} (middle) we show samples from the learned model overlaid on Earth's map and in \Cref{fig:globe} (right) we show a heat map of the learned density.

\begin{figure*}[p]
\centering
\includegraphics[width=0.6\textwidth]{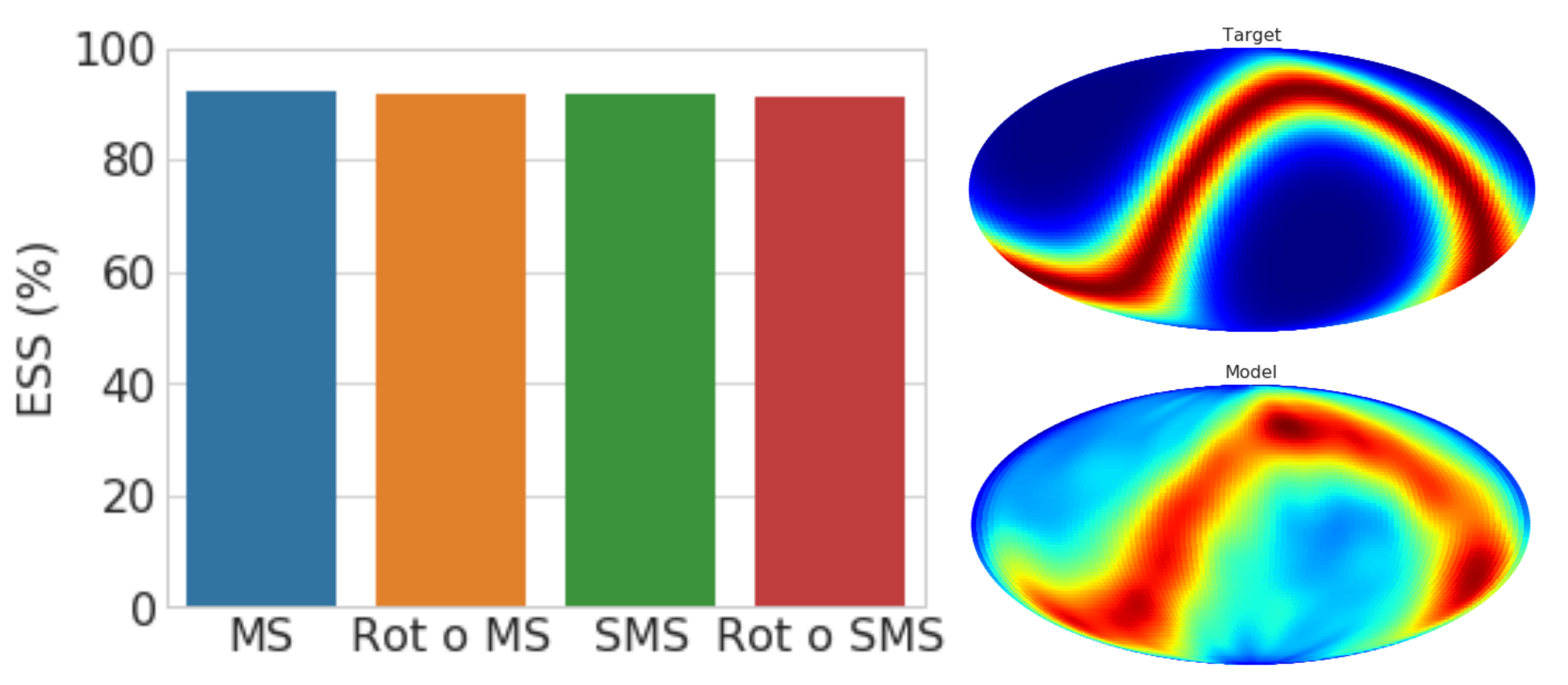}
\caption{Learning a non-axis-aligned density on $\sphere{2}$ using \Cref{eq:flow_sd_1,eq:flow_sd_2,eq:flow_sd_3} with and without composing with a learnable rotation. We compare M\"obius-spline flow (MS) ($K_s=32$, $K_m=12$), learnable rotation composed with MS (Rot $\circ$ MS), small MS ($K_s=3$, $K_m=3$) (SMS) and learnable rotation composed with SMS (Rot $\circ$ SMS)\@. We observed no substantial differences between these models, suggesting that the particular choice of axis inside the flow has no impact on performance.
}\label{fig:slanted}
\end{figure*}

\begin{figure*}[p]
    \centering
    \includegraphics[width=0.6\textwidth]{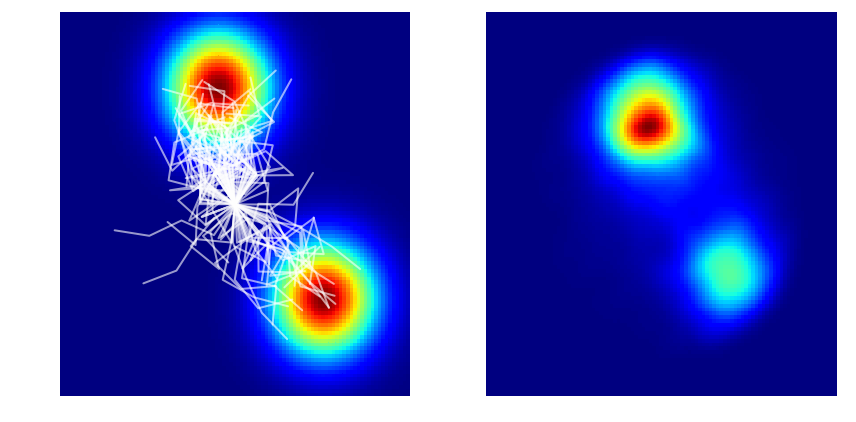}
\caption{Learning the posterior density over joint angles of a $6$-link $2$D robot arm. We used an autoregressive M\"obius flow on the torus $\torus{6}$ to approximate the posterior density of joint angles resulting in a Gaussian mixture density for the tip of the robot arm. {\bf Left}: The heat map shows the target density for the tip of the robot arm, a Gaussian mixture with two modes with centres at $(-0.5, 0.5)$ and $(0.6, -0.1)$. White paths show arm configurations sampled from the learned model in angle space converted to position space. {\bf Right}: Density of the tips of the robot arm using samples from the learned model.}\label{fig:robot_arm}
\end{figure*}

\begin{figure*}[p]
    \centering
    \includegraphics[width=\linewidth]{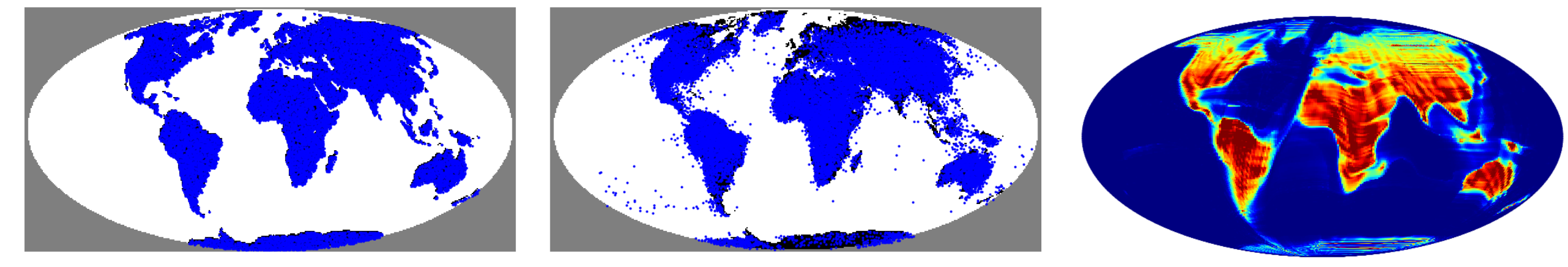}
\caption{Learning a complex density from data samples using autoregressive spline flows. {\bf Left}: Target density from which i.i.d.\ data samples were generated. {\bf Middle}: Model samples overlaid on target density;
{\bf Right}: Heat map of the learned density.}\label{fig:globe}
\end{figure*}

\end{document}